\documentclass{lcs2}

\usepackage{inconsolata}

\usepackage{graphicx}

\PassOptionsToPackage{numbers, compress}{natbib}
\PassOptionsToPackage{dvipsnames, table}{xcolor}

\usepackage{hyperref}       
\usepackage{url}            
\usepackage{booktabs}       
\usepackage{amsfonts}       
\usepackage{nicefrac}       
\usepackage{microtype}      
\usepackage{xcolor}         
\usepackage{microtype}
\usepackage{graphicx}
\usepackage{subcaption}
\usepackage{enumitem}
\usepackage{makecell}
\usepackage{booktabs} 
\usepackage{wrapfig}

\usepackage{hyperref}
\usepackage{amsmath}
\usepackage{amssymb}
\usepackage{mathtools}
\usepackage{amsthm}
\usepackage{booktabs}
\usepackage{multirow}
\usepackage{graphicx}
\usepackage{xcolor}
\usepackage{colortbl}
\usepackage{arydshln}
\usepackage{enumitem}
\usepackage{minitoc}
\usepackage[ruled,vlined, linesnumbered]{algorithm2e}

\setlist[itemize]{noitemsep, topsep=0pt}

\definecolor{blockblue}{RGB}{235, 245, 251}   
\definecolor{blockred}{RGB}{253, 237, 236}    
\definecolor{blockgreen}{RGB}{233, 247, 239}  
\definecolor{blockorange}{RGB}{254, 245, 231} 

\usepackage[capitalize,noabbrev]{cleveref}

\theoremstyle{plain}
\newtheorem{theorem}{Theorem}[section]
\newtheorem{proposition}[theorem]{Proposition}

\theoremstyle{definition}

\newtheorem{assumption}[theorem]{Assumption}
\theoremstyle{remark}

\newcommand{\loss}{\texttt{CaRE-Divergence}}
\newcommand{\method}{\texttt{CaRE-KD}}
\newcommand{\revivalname}{\texttt{Revival}}

\newcommand{\rev}[1]{{\color{black}#1}}

\usepackage[textsize=tiny]{todonotes}

\paperwebsite{https://parmanu.lcs2.in}  

\papertitle{Distilling What Matters: Confidence-Aware Selective Distillation for Large Language Models}

\papershortauthors{A. Sengupta and V. Seth and T. Chakraborty} 
\paperauthors{%
  Ayan Sengupta \textsuperscript{\tiny 1} \quad 
  Vaibhav Seth \textsuperscript{\tiny 1} \quad
  Tanmoy Chakraborty \textsuperscript{\tiny {1,2}}%
} 

\paperaffil{\textsuperscript{1} Department of Electrical Engineering, Indian Institute of Technology Delhi \\ \textsuperscript{2} Yardi School of Artificial Intelligence, Indian Institute of Technology Delhi} 

\paperemails{{\ttfamily{ayan.sengupta@ee.iitd.ac.in, vaibhavseth2283@gmail.com, tanchak@iitd.ac.in}}} 

\paperkeywords{Knowledge distillation, efficient language models, on-policy distillation} 

\paperabstract{ 
  Knowledge Distillation (KD) trains a smaller-capacity student model to imitate a larger-capacity teacher model by matching output distributions, implicitly assuming the teacher to be a reliable oracle. In large language models (LLMs), this assumption often fails: teacher predictions can exhibit high entropy and hallucinations, causing standard KD to degrade well-calibrated student priors. We propose \method, a confidence-gated distillation framework that replaces static objectives with uncertainty-adaptive optimization. \method\ has two components: a token-level loss (\loss) that adaptively switches between Forward and Reverse KL divergence based on teacher--student confidence, and a batch-level epistemic rejection mechanism (\revivalname) that suppresses updates when the teacher is more uncertain than the student. We provide a gradient-level analysis showing how this dual-granularity design induces a conditional calibration mechanism that prior static divergences cannot reproduce. Empirically, across eight teacher--student pairs and eleven benchmarks spanning instruction following, chat alignment, code generation, and mathematical reasoning, \method\ delivers consistent gains over strong baselines (Skewed-KL, $\alpha$--$\beta$ divergence). Highlights include up to $+3.2$ average ROUGE-L on instruction-following tasks, $+2.1$ pass@1 on MBPP, $+1.7$ accuracy on GSM8k, and $+1.8$ accuracy on CollegeMath over the strongest baseline, with consistent gains in LLM-as-a-judge factuality (up to $+2.5$ per task over Skewed-RKL). \revivalname\ further acts as a principled, loss-agnostic plug-in that systematically strengthens existing distillation objectives by filtering epistemically unreliable teacher supervision.
} 

\appendixtocon 
\appendixtocname{Structure of the Appendix}   
\begin{document}

\makelabtitle

\section{Introduction and Prior Art}
\label{sec:intro}

Knowledge Distillation (KD) is a standard paradigm for transferring knowledge from a large teacher model to a small student model by aligning their predictive distributions, most commonly via the Forward Kullback--Leibler (KL) divergence \citep{kim2016sequence, agarwal2024policy}. This formulation implicitly assumes that the teacher provides uniformly reliable supervision, encouraging the student to cover the teacher's full probability mass. While effective when the teacher is well calibrated, this assumption becomes fragile in the regime of large language models (LLMs). Despite their strong average performance, modern LLMs frequently exhibit high-entropy predictions, internal inconsistency, and hallucinations \citep{kalai2025language, sun2025large}. In such cases, minimizing Forward KL forces the student to allocate probability mass to unreliable long-tail predictions, a failure mode often described as zero-forcing \citep{gu2024minillm}, which degrades confident and correct student priors.

\begin{figure*}[htb]
    \centering
    \subfloat[]{\includegraphics[width=0.33\linewidth]{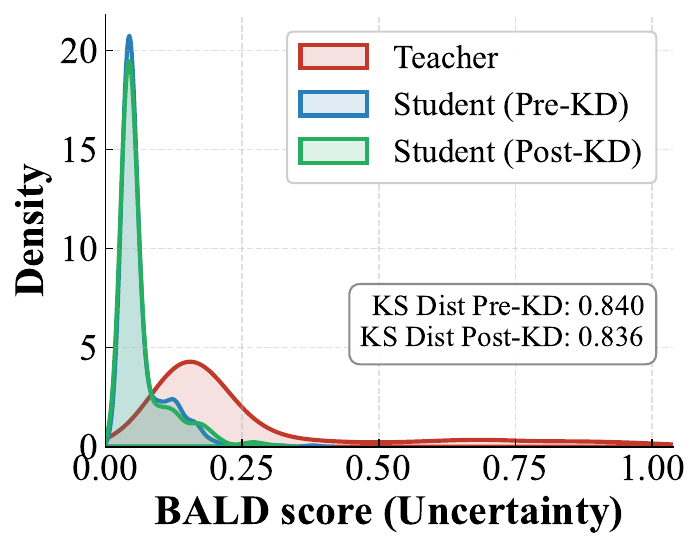}}
    \subfloat[]{\includegraphics[width=0.33\linewidth]{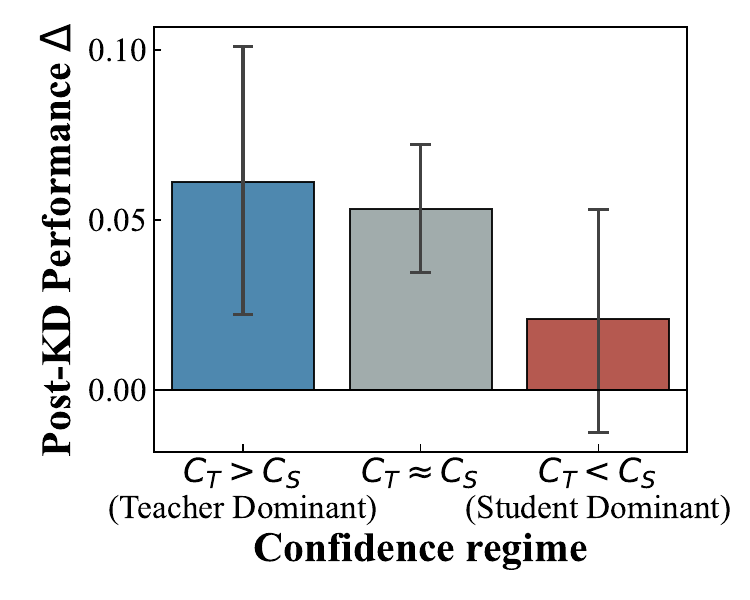}}
    \subfloat[]{\includegraphics[width=0.33\linewidth]{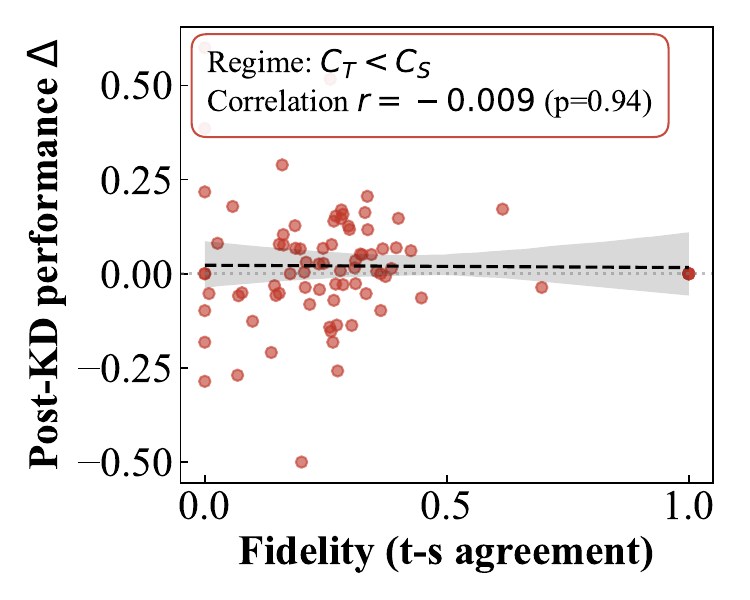}}
    \caption{OpenLLaMA-3B student distilled from a 7B teacher on Dolly-15k with Skewed-RKL \citep{ko2024distillm}. (a) Distillation fails to reduce student epistemic uncertainty when the teacher is uncertain (KS distance changes by only $0.4\%$). (b) Teacher supervision helps when the teacher is confident (median $+4\%$) but barely helps when the teacher is uncertain ($+0.5\%$). (c) Static objectives enforce agreement even with unreliable teachers.}
    \label{fig:motivation}
\end{figure*}

Prior work has attempted to mitigate this issue by modifying the distillation objective -- mode-seeking alternatives such as Reverse KL \citep{gu2024minillm}, as well as static interpolations, including Skewed KL \citep{ko2024distillm} and $\alpha$--$\beta$ divergence \citep{wangabkd}, reduce sensitivity to the teacher's tail by globally biasing optimization toward precision. However, these approaches impose a fixed geometric preference throughout training. They fail to account for the inherently dynamic nature of autoregressive generation, in which a teacher may be highly informative for some tokens while uncertain or hallucinating for others.

Our empirical analysis (Figure~\ref{fig:motivation}) reveals a recurrent failure mode we call the \emph{Fidelity Trap} \citep{stanton2021does, ramesh-etal-2025-generalization}: teacher supervision helps when the teacher is accurate (median $+4\%$) but collapses to $+0.5\%$ when the teacher is uncertain about its own generation \citep{kadavath2022languagemodelsmostlyknow, prato-etal-2024-large}, even though the student often has \emph{lower} pre-distillation uncertainty than the teacher in such regimes. Standard KD nevertheless forces the student to match the teacher, synchronizing with unreliable supervision and overwriting correct priors.

We propose \method, with two complementary components. \loss\ is a token-level objective that dynamically switches between Forward KL (mean-seeking, when the teacher is confident) and Reverse KL (mode-seeking, when the teacher is uncertain). \revivalname\ is a batch-level rejection mechanism based on BALD \citep{houlsby2011bayesian} that suppresses updates when the teacher is epistemically less reliable than the student, preventing unlearning of correct priors.

Empirically, \method\ achieves consistent improvements across eight teacher--student pairs and eleven benchmarks spanning instruction following, chat alignment, code generation, and mathematical reasoning (Tables~\ref{tab:main_results}, \ref{tab:divergence_comparison}, \ref{tab:revival_delta}, and~\ref{tab:additional_results}). Concrete highlights include a $+2.6$ average ROUGE-L gain on GPT2-base over Forward KL, $+3.2$ on OPT-125M, $+2.1$ pass@1 on MBPP, $+1.7$ accuracy on GSM8k, and consistent gains in LLM-as-a-judge factuality (up to $+2.5$ per task over Skewed-RKL). Moreover, \revivalname\ acts as a general enhancer: when applied to existing divergence objectives, including Forward KL, Reverse KL, $\alpha$--$\beta$ divergence, and Skewed RKL, it systematically improves performance by selectively filtering unreliable supervision (Table~\ref{tab:revival_delta}). Together, these results demonstrate that selective adaptation of optimization geometry and selective rejection of supervision are both necessary for robust distillation from modern LLMs.
Unlike prior work that modifies divergence globally or filters data at the dataset level, our framework adapts both the optimization geometry at the token level and supervision at the sequence level during training.

Our contributions can be summarized as follows:
\begin{itemize}[leftmargin=*, itemsep=0pt, parsep=0pt, topsep=0pt]
    \item We introduce \loss, a confidence-gated distillation objective that dynamically adapts divergence geometry at the token level.
    \item We propose \revivalname, a sequence-level epistemic rejection mechanism that filters distillation updates when the teacher is epistemically less reliable than the student, preventing student calibration from collapsing under unreliable supervision.
    \item We provide a gradient-level analysis explaining \emph{how} \loss\ induces conditional calibration and \emph{why} no static divergence can reproduce its token-wise behavior, together with formal conditions under which entropy-based confidence is a valid proxy for BALD.
    \item We demonstrate consistent empirical gains across instruction following, chat alignment, code generation, and mathematical reasoning, complemented by an LLM-as-a-judge factuality evaluation that directly measures hallucination, and release a unified evaluation framework for reproducibility.\footnote{The supplementary material contains source code for all divergence objectives implemented in this paper.}
\end{itemize}


\noindent \textbf{Prior Art.}
Knowledge Distillation has evolved from logit matching and sequence-level objectives \citep{hinton2015distilling, kim2016sequence} to generative distillation frameworks that address exposure bias via on-policy sampling \citep{agarwal2024policy, sengupta2024a} and mitigate tail risks \citep{gu2024minillm}. Despite these advances, robustly handling noisy or hallucinated supervision remains an open challenge \citep{deepseekai2025deepseekr1incentivizingreasoningcapability, ramesh-etal-2025-generalization, fang2024bayesian}. While prior work has explored static reweighting of divergence geometry \citep{ko2024distillm, wangabkd}, proxy-based uncertainty modeling \citep{fang2024bayesian}, or dataset-level filtering \citep{he2025dakd}, these approaches either introduce additional modeling overhead or completely discard potentially informative data. In contrast, our method directly modulates the optimization trajectory, leveraging divergence geometry and intrinsic uncertainty signals to adapt supervision at training time without auxiliary models or data pruning. A more detailed discussion of related work is provided in Appendix~\ref{sec:related}.

\section{Methodology}
\label{sec:methodology}

\begin{wrapfigure}{r}{0.55\textwidth} 
    \centering
    \includegraphics[width=\linewidth]{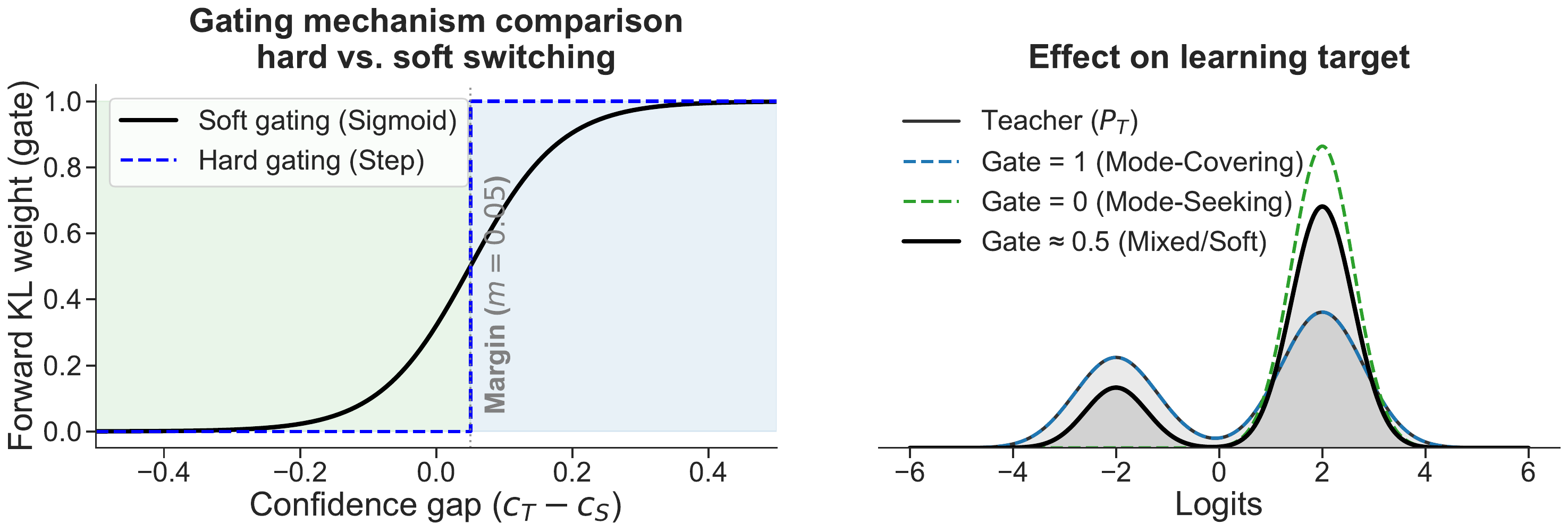}
    \caption{\textbf{Token-level geometric adaptation induced by \loss.} When the teacher is confident, the objective behaves as Forward KL (mean-seeking, recall-oriented); when the teacher is uncertain, it transitions to Reverse KL (mode-seeking, precision-oriented), suppressing tail noise.}
    \label{fig:method}
\end{wrapfigure}

\method\ is a dual-granularity distillation framework with two complementary mechanisms: (i) a token-level geometric adaptation, \loss\ (\S\ref{subsec:care_loss}), that switches between Forward and Reverse KL using single-pass confidence, and (ii) a sequence-level rejection rule, \revivalname\ (\S\ref{subsec:revival}), that suppresses updates when the teacher is epistemically less reliable than the student.

\noindent \textbf{Notation.} Let $\mathcal{V}$ denote the vocabulary, and let $p(\cdot\mid\mathbf{x}), q(\cdot\mid\mathbf{x}) \in \Delta^{|\mathcal{V}|-1}$ be the teacher and student next-token distributions (optionally tempered by $T$). We measure single-pass (aleatoric) confidence by the normalized-entropy score
\begin{equation}
\label{eq:conf_score}
c(p) \triangleq 1 - H(p)/\log|\mathcal{V}| \in [0,1],
\end{equation}
where $H(\cdot)$ is Shannon entropy; larger $c(p)$ corresponds to a sharper distribution.

\subsection{Token-Level Dynamics: \loss}
\label{subsec:care_loss}

Since the appropriate optimization geometry depends on the relative informativeness of the teacher (Figure~\ref{fig:method}), we replace fixed Forward KL with the confidence-gated combination
\begin{equation}
\label{eq:care_objective}
\mathcal{L}_{\mathrm{CARE}}(p, q) = g \cdot \mathrm{KL}(p \| q) + (1 - g) \cdot \mathrm{KL}(q \| p),
\end{equation}
where $g\in[0,1]$ is a per-token scalar gate driven by the confidence gap $\Delta c = c(p)-c(q)$. We consider two variants:
\begin{equation}
g_{\mathrm{hard}} = \mathbb{I}[c(p) \ge c(q)],\qquad g_{\mathrm{soft}} = \sigma\!\big(c(p) - c(q) - m\big),
\label{eq:gates}
\end{equation}
with $\sigma(\cdot)$ the sigmoid and $m$ a margin. When the teacher is confident ($g\to1$), \loss\ approaches mean-seeking $\mathrm{KL}(p\|q)$ and maximizes support coverage; when the teacher is uncertain ($g\to0$), it approaches mode-seeking $\mathrm{KL}(q\|p)$, suppressing tail noise.

\subsection{Batch-Level Rejection: \revivalname}
\label{subsec:revival}

Single-pass entropy cannot distinguish intrinsic ambiguity from hallucination-driven epistemic uncertainty. We therefore introduce \revivalname, a sequence-level rejection mechanism based on the BALD score \citep{houlsby2011bayesian} computed via Monte Carlo dropout \citep{gal2016dropout}:
\begin{equation}
\label{eq:bald_score}
\mathcal{I}_{\mathrm{BALD}}(\mathbf{x}) = H\!\left[\mathbb{E}_{\omega} p(y\mid\mathbf{x},\omega)\right] - \mathbb{E}_{\omega}\!\left[H\!\left(p(y\mid\mathbf{x},\omega)\right)\right],
\end{equation}
which measures predictive disagreement across stochastic parameter realizations. Let $\mathcal{I}_{\mathrm{T}}$ and $\mathcal{I}_{\mathrm{S}}$ denote teacher and student BALD scores, respectively at a sequence; the rejection mask is
\begin{equation}
\label{eq:revival_mask}
M_{\mathrm{Revival}} = \mathbb{I}\!\left[\mathcal{I}_{\mathrm{T}} > Q_{\tau}(\mathcal{I}_{\mathrm{T}}) \;\land\; \mathcal{I}_{\mathrm{S}} < Q_{\tau}(\mathcal{I}_{\mathrm{S}})\right],
\end{equation}
where $Q_\tau(\cdot)$ is the running $\tau$-quantile of historical BALD scores (computed separately for teacher and student). Using a relative percentile rather than a fixed threshold ensures the mask is scale-invariant under training-time drift in BALD magnitudes, and triggers \emph{only} when the teacher is uncertain \emph{relative} to a confident student, targeting unreliable supervision rather than uniformly discarding hard tokens. We optionally schedule $\tau$ across training (Appendix~\ref{app:exp_details}).


For a batch $B$, the final per-batch objective is
\begin{equation}
\label{eq:final_loss}
\mathcal{L}_{\mathrm{CaRE\text{-}KD}} = (1-M_{\mathrm{Revival}}) \cdot \frac{1}{|B|}\sum_{\mathbf{x}\in B}\mathcal{L}_{\mathrm{CARE}}\big(p(\mathbf{x}), q(\mathbf{x})\big).
\end{equation}
When $M_{\mathrm{Revival}}=1$, updates are suppressed; otherwise the student optimizes the gated objective.

\section{Theoretical Results}
\label{sec:theory}

We provide a gradient-level analysis of \loss\ that explains \emph{how} the gated objective differs from Skewed KL \citep{ko2024distillm,kodistillm} and $\alpha$--$\beta$ divergence \citep{wangabkd}, and we formalize when entropy-based confidence proxies BALD. Our results are mechanistic explanations of the empirical behavior, not standalone optimization guarantees. All proofs are in Appendix~\ref{app:proofs}.

\subsection{Gradient analysis of \loss}

\begin{proposition}[Gradient decomposition of \loss]
\label{thm:grad_decomp}
Let $\mathcal{L}_{\mathrm{FKL}}(\theta)=\mathrm{KL}(p\|q_\theta)$ and $\mathcal{L}_{\mathrm{RKL}}(\theta)=\mathrm{KL}(q_\theta\|p)$ with $p$ fixed in $\theta$, and let $g=g(\theta)$ be a differentiable gate. Then
\begin{multline}
\label{eq:grad_decomp}
\nabla_\theta \mathcal{L}_{\mathrm{CaRE}}
=
\underbrace{\Big[g \nabla_\theta \mathcal{L}_{\mathrm{FKL}} + (1-g)\nabla_\theta \mathcal{L}_{\mathrm{RKL}}\Big]}_{\text{Learning signal}}
+
\underbrace{\Big[(\mathcal{L}_{\mathrm{FKL}}-\mathcal{L}_{\mathrm{RKL}})\nabla_\theta g\Big]}_{\text{Selection signal}}.
\end{multline}
For the soft gate $g=\sigma(c(p)-c(q_\theta)-m)$, $\nabla_\theta g = -g(1-g)\nabla_\theta c(q_\theta)$, so the selection signal equals $-g(1-g)(\mathcal{L}_{\mathrm{FKL}}-\mathcal{L}_{\mathrm{RKL}})\nabla_\theta c(q_\theta)$.
\end{proposition}

\noindent \textbf{Interpretation.} When the teacher is comparatively uncertain ($\mathcal{L}_{\mathrm{FKL}}>\mathcal{L}_{\mathrm{RKL}}$), the selection signal aligns with $+\nabla_\theta c(q_\theta)$, explicitly pushing the student to \emph{increase} its confidence. No fixed-skew or fixed-$\beta$ objective produces such a confidence-dependent term. \rev{This selection term is present only for the differentiable soft gate: the hard gate $g_{\mathrm{hard}}$ is a detached step function with $\nabla_\theta g\equiv0$, so it retains only the token-level branch selection (the learning term with per-token $g_t\in\{0,1\}$, which Proposition~\ref{thm:skewed_rkl_relation} shows no static divergence can reproduce), while soft gating additionally carries the confidence-pushing selection term. Branch selection alone yields the strongest peak performance, whereas the selection term confers robustness (\Cref{app:proof_grad}).} When $M_{\mathrm{Revival}}=1$, $\nabla_\theta\mathcal{L}_{\mathrm{CaRE\text{-}KD}}=0$, so the divergence to any reference $p^\star$ is unchanged -- \revivalname\ provides worst-case robustness against unreliable supervision (proof in Appendix~\ref{app:proofs}).

\subsection{Strict separation from SRKL and $\alpha$--$\beta$ divergence}

\begin{proposition}[No global skew or $\beta$ matches \loss]
\label{thm:skewed_rkl_relation}
In the student-dominant regime ($c(q)>c(p)$, $g=0$), the per-token \loss\ gradient matches that of Skewed Reverse KL in the limit $\lambda\to 1$. Globally, however, \loss\ assigns a token-varying $g_t\in\{0,1\}$ (and equivalently a token-varying $\beta_t\in\{0,1\}$ in the $\alpha$--$\beta$ family). No global $\lambda$ for SRKL or global $\beta$ for $\alpha$--$\beta$ divergence reproduces this token-wise behavior. (Full $\alpha$--$\beta$ derivation: Appendix~\ref{app:proofs}.)
\end{proposition}

\subsection{Entropy-confidence proxies BALD in the sharp regime}

\begin{assumption}[Sharpness Hypothesis]
\label{ass:sharpness}
LLMs trained with cross-entropy loss are typically overconfident \citep{sun2025large,pawitan2025confidence}: $H[p(y|\mathbf{x}, \omega)] \approx 0$ for $\omega\sim p(\omega)$.
\end{assumption}

\begin{theorem}[Conditional equivalence of confidence proxies]
\label{thm:conf_equiv}
Under Assumption~\ref{ass:sharpness} and the approximation $\hat p \approx \mathbb{E}_\omega[p(y\mid\mathbf{x},\omega)]$, maximizing $c(\hat p)$ approximately minimizes $\mathcal{I}_{\mathrm{BALD}}(\mathbf{x})$. Hence single-pass entropy is a computationally efficient proxy for BALD inside the inner loop.
\end{theorem}

\subsection{Conditional bi-directional calibration}

Within a fixed branch of the gated objective, both KL branches share the unique minimizer $q=p$ (under standard support conditions); convergence in a branch, therefore, drives student entropy toward $H(p)$.

\begin{theorem}[Conditional bi-directional calibration]
\label{thm:conditional_calibration_overconfident}
\label{thm:conditional_sharpening_underconfident}
Suppose optimization under $\mathcal{L}_{\mathrm{CaRE\text{-}KD}}$ stays in a fixed branch regime and $q_\theta\to p$. Then $H(q_{\theta_t})\to H(p)$. In particular, an overconfident student ($H(q_{\theta_0})<H(p)$) must increase entropy at some step, and an underconfident student ($H(q_{\theta_0})>H(p)$) must decrease entropy. \loss\ thus acts as a conditional calibration mechanism without requiring strict monotonicity.
\end{theorem}

\section{Experimental Setup and Results}
\label{sec:exp}

Following established protocols for evaluating distilled language models \citep{kodistillm, wangabkd, ko2024distillm}, we evaluate our framework across three modalities -- instruction following, code generation, and mathematical reasoning.

\noindent \textbf{Instruction Following.}
We consider two complementary evaluation regimes.
(1) \textbf{Standard benchmarks.}
Following \citet{ko2024distillm}, we perform distillation on the Dolly-15k dataset \citep{conover2023free}. We evaluate four teacher--student model families: (i) GPT-2$_{\text{XL}} \rightarrow$ GPT-2$_{\text{Base/Large}}$ \citep{radford2019language}, (ii) OPT-2.7B $\rightarrow$ OPT-125M \citep{zhang2022opt}, (iii) Gemma-2-9B-Instruct $\rightarrow$ Gemma-2-2B-Instruct~\citep{team2024gemma} and (iv) OpenLLaMA-7B $\rightarrow$ OpenLLaMA-3B \citep{openlm2023openllama}. Evaluation is conducted on four task-agnostic instruction-following benchmarks: Dolly Evaluation \citep{conover2023free}, Self-Instruct \citep{wang2023self}, Super-Natural Instructions \citep{wang-etal-2022-super}, and Vicuna Instructions \citep{wang-etal-2022-super}.
(2) \textbf{Large-scale alignment.}
To assess performance on modern chat-aligned models, we follow \citet{kodistillm} and distill a Qwen2-1.5B student from a Qwen2-7B teacher \citep{qwen2} using the UltraChat200k dataset \citep{ding2023enhancing}. The resulting models are evaluated on AlpacaEval \citep{li2023alpacaeval}, Evol-Instruct \citep{xu2024wizardlm}, and UltraFeedback \citep{cuiultrafeedback}.

\noindent \textbf{Code Generation.}
For domain-specific coding tasks, we replicate the setup of \citet{kodistillm}. We distill a DeepSeek-Coder-1.3B student from a DeepSeek-Coder-6.9B-Instruct teacher \citep{guo2024deepseekcoderlargelanguagemodel}. Training is performed on the WizardCoder dataset \citep{luo2024wizardcoder}, and evaluation is conducted on HumanEval \citep{chen2021evaluatinglargelanguagemodels} and MBPP \citep{austin2021program}.

\noindent \textbf{Mathematical Reasoning.}
To evaluate multi-step reasoning capabilities, we use the Qwen2.5-Math models \citep{qwen2.5}. A 1.5B student is distilled from a 7B-Instruct teacher using the MetaMathQA dataset \citep{yumetamath}. We report zero-shot accuracy on GSM8K \citep{cobbe2021training} and CollegeMath \citep{tang2024mathscale}.

\noindent \textbf{Metrics.}
For instruction-following and chat alignment we report ROUGE-L on the held-out evaluation set, following \citet{ko2024distillm}, together with an LLM-as-a-judge factuality score (0--100) computed by prompting GPT-5-Mini to grade outputs against the reference answer for hallucinations and factual consistency (system prompt in Appendix~\ref{app:llm_judge_prompt}). For code generation we report execution-based pass@1 on HumanEval and MBPP, and for mathematical reasoning we report exact-match accuracy on GSM8K and CollegeMath. Statistical significance for ROUGE-L improvements is assessed using paired $t$-tests across model pairs. Additional implementation details, hyperparameters, training protocols, and the LLM-judge system prompt are provided in Appendix~\ref{app:exp_details}.

\subsection{Results on instruction following tasks}
\label{subsec:main_results}

We evaluate \method\ under SGO training (student rollouts) and non-SGO (teacher forcing); each cell of Tables~\ref{tab:main_results} and~\ref{tab:divergence_comparison} reports ROUGE-L / LLM-as-a-judge factuality.

In the SGO setting (Table~\ref{tab:main_results}), \loss\ improves over Forward KL by $+2.6$ / $+4.3$ on GPT2-base, $+2.0$ / $+3.4$ on GPT2-large, and $+3.2$ / $+3.0$ on OPT-125M (ROUGE-L / LLM-judge), and beats Skewed RKL by $+0.3$ / $+1.5$ on GPT2-base. On OpenLLaMA-7B $\rightarrow$ 3B, it achieves the best ROUGE-L on Self-Instruct and Super-Natural Instructions and stays within 0.2 of the best baseline on average. Paired $t$-tests confirm that the average improvements of \loss\ over the per-row second-best baseline are statistically significant ($t=3.93$, $p=0.03$). Detailed comparisons against a comprehensive baseline suite (SFT, KD, SeqKD, ImitKD, MiniLLM, GKD, Distillm, ABKD) appear in Appendix~\ref{app:detailed_results} (Tables~\ref{tab:baseline_gpt2base},~\ref{tab:baseline_opt},~\ref{tab:baseline_llama}); \loss\ tops the average ROUGE-L on GPT2-base and OPT-125M.

\begin{table*}[t]
\caption{\textbf{Main instruction-following results under SGO training.} Cells report ROUGE-L / LLM-as-a-judge factuality on Dolly, Self-Instruct, Super-Natural Instructions, and Vicuna. Across five teacher--student pairs, \loss\ achieves the best average ROUGE-L on three pairs and the best factuality on three pairs, with statistically significant gains over the second-best baseline ($t=3.93$, $p=0.03$). Bold denotes the best score per cell. Results are averaged over five seeds; detailed comparisons appear in Appendix~\ref{app:detailed_results}.}
\label{tab:main_results}
\centering
\begin{small}
\begin{sc}
\scalebox{0.85}{
\begin{tabular}{llccccc}
\toprule

\multirow{2}{*}{\textbf{Teacher $\rightarrow$ Student}} & \multirow{2}{*}{\textbf{Loss}} & \multicolumn{4}{c}{\textbf{Benchmarks}} & \multirow{2}{*}{\textbf{Avg}} \\
\cmidrule(lr){3-6}
  & & Dolly & Self-Inst & Super-Nat & Vicuna & \\
\midrule

\rowcolor{blockblue}  & FKL & 23.4/21.2 & 10.7/15.6 & 19.0/16.4 & 15.0/13.5 & 17.0/16.7 \\
\rowcolor{blockblue}  & RKL & 24.7/20.7 & 11.6/16.4 & 19.2/16.2 & 16.6/14.6 & 18.0/17.0 \\
\rowcolor{blockblue}  & SRKL & 24.9/23.1 & 12.1/17.6 & \textbf{23.6}/21.9 & 16.5/15.3 & 19.3/19.5 \\
\rowcolor{blockblue}  & $\alpha$-$\beta$ Div & 23.9/23.1 & 11.5/16.4 & 21.9/18.3 & 16.1/15.0 & 18.4/18.2 \\
\rowcolor{blockblue} \multirow{-5}{*}{\shortstack[l]{GPT2-XL (1.5B) $\rightarrow$ \\ GPT2-base (0.1B)}}
& \texttt{CaRE-Div} (Ours)& \textbf{25.6/24.1} & \textbf{12.9/20.1} & 22.8/\textbf{24.4} & \textbf{17.3/15.3} & \textbf{19.6/21.0} \\

\cdashline{1-7}

\rowcolor{blockred}  & FKL & 23.7/21.5 & 10.1/13.6 & 15.7/13.4 & 15.0/14.4 & 16.1/15.7 \\
\rowcolor{blockred}  & RKL & 19.3/14.7 & 9.9/14.3 & 16.3/13.0 & 15.0/11.3 & 15.1/13.3 \\
\rowcolor{blockred}  & SRKL & 24.1/20.5 & 11.1/15.0 & \textbf{19.9}/17.0 & \textbf{16.4}/15.3 & 17.9/16.9 \\
\rowcolor{blockred}  & $\alpha$-$\beta$ Div & 23.0/20.1 & 10.7/14.6 & 16.7/15.2 & 15.5/13.3 & 16.5/15.8 \\
\rowcolor{blockred} \multirow{-5}{*}{\shortstack[l]{GPT2-XL (1.5B) $\rightarrow$ \\ GPT2-large (0.8B)}}
& \texttt{CaRE-Div} (Ours)& \textbf{25.4/25.5} & \textbf{11.6/16.7} & 19.4/\textbf{18.0} & 16.0/\textbf{16.1} & \textbf{18.1/19.1} \\

\cdashline{1-7}

\rowcolor{blockorange}  & FKL & 22.4/20.7 & 9.6/14.1 & 17.1/14.5 & 15.5/12.9 & 16.1/15.5 \\
\rowcolor{blockorange}  & RKL & 24.8/22.7 & 11.2/15.8 & 22.2/18.4 & 16.4/15.9 & 18.6/18.2 \\
\rowcolor{blockorange}  & SRKL & \textbf{25.1/23.9} & 11.6/\textbf{19.3} & 22.5/19.5 & \textbf{16.8/18.1} & 19.0/\textbf{20.2} \\
\rowcolor{blockorange}  & $\alpha$-$\beta$ Div & 24.9/23.6 & 11.2/16.5 & 21.7/18.3 & 16.7/13.9 & 18.6/18.1 \\
\rowcolor{blockorange} \multirow{-5}{*}{\shortstack[l]{OPT-2.7B $\rightarrow$ \\ OPT-125M}}
& \texttt{CaRE-Div} (Ours)& 24.3/20.5 & \textbf{12.0}/18.7 & \textbf{23.9/21.4} & 16.7/13.5 & \textbf{19.3}/18.5 \\

\cdashline{1-7}

\rowcolor{blockred}  & FKL & 18.4/19.6 & 10.4/22.0 & 19.2/15.3 & 17.1/16.0 & 16.3/18.2 \\
\rowcolor{blockred}  & RKL & 17.5/20.2 & 10.0/21.8 & \textbf{19.7}/19.1 & 20.1/\textbf{21.9} & 16.8/20.7 \\
\rowcolor{blockred}  & SRKL & \textbf{19.7}/24.9 & 10.8/20.8 & 16.5/18.3 & 18.9/18.2 & 16.5/20.6 \\
\rowcolor{blockred}  & $\alpha$-$\beta$ Div & 18.9/\textbf{25.6} & 10.7/21.6 & 17.5/17.2 & 20.0/20.2 & 16.8/21.1 \\
\rowcolor{blockred} \multirow{-5}{*}{\shortstack[l]{Gemma-2-9B-IT $\rightarrow$ \\ Gemma-2-2B-IT}}
& \texttt{CaRE-Div} (Ours)& 17.0/23.4 & \textbf{11.5/22.0} & 19.4/\textbf{19.9} & \textbf{20.3}/21.0 & \textbf{17.1/21.6} \\

\cdashline{1-7}

\rowcolor{blockblue}  & FKL & 24.3/32.2 & 17.6/25.7 & 30.5/31.3 & 16.9/22.3 & 22.3/27.9 \\
\rowcolor{blockblue}  & RKL & \textbf{29.1}/40.0 & 20.7/\textbf{31.7} & 35.7/\textbf{39.4} & \textbf{20.0}/27.8 & \textbf{26.4}/34.7 \\
\rowcolor{blockblue}  & SRKL & 29.0/\textbf{40.9} & 20.8/31.2 & 36.6/37.8 & 19.2/\textbf{31.1} & \textbf{26.4/35.3} \\
\rowcolor{blockblue}  & $\alpha$-$\beta$ Div & 28.3/38.1 & 19.3/29.7 & 33.5/35.2 & 18.6/29.9 & 24.9/33.2 \\
\rowcolor{blockblue} \multirow{-5}{*}{\shortstack[l]{OpenLLaMA2-7B $\rightarrow$ \\ OpenLLaMA2-3B}}
& \texttt{CaRE-Div} (Ours)& 27.7/38.1 & \textbf{20.9}/29.8 & \textbf{37.3}/38.6 & 18.9/28.6 & 26.2/33.8 \\

\bottomrule
\end{tabular}
}
\end{sc}
\end{small}
\end{table*}

\rev{\noindent\textbf{The advantage is robustness, not a single large win.} The accurate reading of Table~\ref{tab:main_results} is that \method\ is ``never bad'' rather than ``always best'': under SGO it has the best average on four of the five pairs and the highest mean ROUGE-L averaged over all five ($20.1$ vs.\ $19.8$ for the next-best static objective, Skewed-RKL). The more informative comparison is worst-case behavior. Each single-geometry objective breaks somewhere --- Forward KL trails the best by $4.1$ on OpenLLaMA, Reverse KL by $3.0$ on GPT2-large, and $\alpha$--$\beta$ by $1.6$ on GPT2-large --- and even the strongest static baseline, Skewed-RKL, trails by up to $0.6$ (Gemma). \method's worst gap on any pair is $0.2$. On OpenLLaMA-7B$\rightarrow$3B the $0.18$ ROUGE-L gap ($26.23$ vs.\ $26.41$, Table~\ref{tab:baseline_llama}) is smaller than the per-seed standard deviation of either method on this pair ($0.19$--$0.90$ across tasks), so we treat the pair as a statistical tie; Gemma-2-2B-IT on Dolly ($17.0$ vs.\ $19.7$) is an outright loss, which we state plainly.}

\rev{\noindent\textbf{Scope of applicability.} Two factors set how much the confidence gate can help: how unreliable the teacher is, and how far the student sits from the teacher (roughly the compression ratio). The clearest wins occur where both are large --- small students distilled from weak teachers under heavy compression (GPT2-XL$\rightarrow$base, $\sim$$15\times$; OPT-2.7B$\rightarrow$125M, $\sim$$22\times$). The one tie, OpenLLaMA-7B$\rightarrow$3B, distills a strong, well-calibrated teacher at only $2.3\times$, so the student already sits close to the teacher and there is little unreliable supervision to filter; here a static mode-seeking objective is near-optimal. \method\ should therefore be preferred under large capacity gaps and precision-sensitive tasks.}

In the non-SGO (teacher-forcing) setting (Table~\ref{tab:divergence_comparison}, Appendix~\ref{app:detailed_results}), \loss\ achieves $19.5$/$19.7$ (R/LLM) on GPT2-base, strictly dominating SRKL ($17.7$/$17.7$) and FKL ($16.1$/$15.6$); validation-set cross-entropy and exact-match dynamics confirming faster, lower-loss convergence are in Figures~\ref{fig:valid_loss}--\ref{fig:valid_exact_match} (Appendix~\ref{app:detailed_results}).

\subsection{LLM-as-a-judge factuality}
\label{subsec:llm_judge}

ROUGE-L is blind to hallucination. Therefore, we additionally grade every output with GPT-5-Mini on a $0$--$100$ factuality scale (system prompt: Appendix~\ref{app:llm_judge_prompt}). \loss\ improves average factuality alongside ROUGE-L on most pairs: $+1.5$ over SRKL on GPT2-base ($21.0$ vs.\ $19.5$), $+2.2$ on GPT2-large, and the highest Evol-Instruct factuality on Qwen2-1.5B (\,$53.0$ vs.\ $52.1$ Distillm-2). Per-task gaps reach $+2.5$ over SRKL (Self-Instruct, Super-Natural Instructions; GPT2-base). To probe whether \revivalname\ targets hallucination-prone regions, we measure teacher factuality conditional on the mask: on Dolly with the OpenLLaMA-7B teacher, the teacher's factuality is $25.2$ when \revivalname\ fires vs.\ $28.6$ when it does not (one-sided Wilcoxon $p=0.06$). \rev{We treat this as suggestive rather than conclusive: the Wilcoxon test is not significant, and a complementary check that splits the sequences by confidence regime and compares teacher factuality where \revivalname\ fires against the opposite region is also not significant (Mann--Whitney $p=0.26$). We therefore do not claim that \revivalname\ selects specifically low-factuality tokens at the mechanism level; the claim we make, and that holds, is that the distilled student's \emph{end-to-end} factuality improves under the LLM judge (Tables~\ref{tab:main_results} and~\ref{tab:additional_results}).}

\subsection{\revivalname\ as a loss-agnostic enhancer}
\label{subsec:revival_impact}

Table~\ref{tab:revival_delta} isolates \revivalname\ by reporting $\Delta = \text{Score}_{\text{with}}-\text{Score}_{\text{without}}$. \loss\ exhibits the most consistent positive response: $+1.42$ avg.\ ROUGE-L on OPT-125M ($+3.62$ on Super-Natural Instructions). \revivalname\ also improves static baselines (e.g., $+1.14$ for RKL on GPT2-base), supporting the broader hypothesis that selective filtering is a general enhancement mechanism \citep{he2025dakd,fang2024bayesian}; however, static objectives are high-variance (RKL $-1.78$ on Self-Instruct for GPT2-large), whereas \loss\ remains consistently positive on average, highlighting the complementarity of dynamic geometry and selective rejection. \rev{The precise claim is about the average, not every cell: \loss\ itself has two small negative entries ($-0.4$ on Vicuna/Gemma, $-0.04$ on Dolly/GPT2-large), but it has the most reliable positive response overall --- the highest mean $\Delta$ with the fewest and smallest negatives. A one-sample $t$-test over the per-task deltas confirms this: \loss\ $t=3.98$ ($p<0.001$), whereas Forward KL ($t=-2.55$) and Reverse KL ($t=1.24$, $p=0.25$) are not reliably positive.}

\begin{table}[!t]
\caption{\textbf{Impact of \revivalname\ on different distillation losses.} Values are the per-task delta $\Delta = \text{Score}_{\text{with-\revivalname}} - \text{Score}_{\text{no-\revivalname}}$ in ROUGE-L. \loss\ shows the most consistent positive response, with a statistically significant average improvement across model pairs (one-sample $t$-test: $t=3.98$, $p<0.001$). Static baselines such as FKL ($t=-2.55$, $p=0.06$) often degrade with \revivalname, while RKL ($t=1.24$, $p=0.25$) yields insignificant gains, indicating that our confidence-gated geometry is uniquely synergistic with epistemic filtering. Bold marks the largest positive delta per (model, task) cell.}
\label{tab:revival_delta}
\centering
\begin{small}
\begin{sc}
\scalebox{0.9}{
\begin{tabular}{llccccc}
\toprule
\multirow{2}{*}{\textbf{Student}} & \multirow{2}{*}{\textbf{Loss}} & \multicolumn{4}{c}{\textbf{Benchmarks}} & \multirow{2}{*}{\textbf{Avg}} \\
\cmidrule(lr){3-6}
  & & Dolly & Self-Inst & Super-Nat & Vicuna & \\
\midrule
\rowcolor{blockblue}  & FKL & 0.22 & 0.14 & -0.64 & 0.09 & -0.05 \\
\rowcolor{blockblue}  & RKL & 0.32 & 0.61 & {3.51} & 0.12 & {1.14} \\
\rowcolor{blockblue}  & SRKL & 0.41 & 0.54 & 1.64 & 0.30 & 0.72 \\
\rowcolor{blockblue}  & $\alpha$-$\beta$ Div  & {1.11} & 0.80 & 0.94 & 0.67 & 0.88 \\
\rowcolor{blockblue} \multirow{-5}{*}{\shortstack[l]{GPT2-base}} & \texttt{CaRE} & 0.35 & {0.94} & 1.26 & {1.04} & 0.90 \\

\cdashline{1-7}

\rowcolor{blockred}  & FKL & -0.36 & -0.69 & -0.22 & -0.26 & -0.38 \\
\rowcolor{blockred}  & RKL & 0.35 & -1.78 & -1.44 & {1.75} & -0.28 \\
\rowcolor{blockred}  & SRKL & {0.43} & 0.70 & {1.09} & -0.56 & 0.41 \\
\rowcolor{blockred}  & $\alpha$-$\beta$ Div  & -0.44 & -0.35 & 0.70 & 0.68 & 0.15 \\
\rowcolor{blockred} \multirow{-5}{*}{\shortstack[l]{GPT2-large}} & \texttt{CaRE} & -0.04 & {0.71} & 1.06 & 0.10 & {0.46} \\

\cdashline{1-7}

\rowcolor{blockorange}  & FKL & -0.23 & -0.18 & -0.01 & -0.29 & -0.18 \\
\rowcolor{blockorange}  & RKL & 0.29 & 0.57 & 0.75 & {0.71} & 0.58 \\
\rowcolor{blockorange}  & SRKL & 0.27 & 0.61 & 1.98 & -0.76 & 0.52 \\
\rowcolor{blockorange}  & $\alpha$-$\beta$ Div  & -0.42 & 0.35 & 0.05 & -0.43 & -0.11 \\
\rowcolor{blockorange} \multirow{-5}{*}{\shortstack[l]{OPT-125M}} & \texttt{CaRE} & {0.55} & {1.12} & {3.62} & 0.38 & {1.42} \\

\cdashline{1-7}

\rowcolor{blockred}  & FKL & -4.02 & -3.68 & -7.31 & -0.66 & -3.91 \\
\rowcolor{blockred}  & RKL & -0.48 & -2.68 & 1.54  & 0.4   & -0.3\\
\rowcolor{blockred}  & SRKL & 1.48  & 1.89  & 1.56  & {0.99}  & 1.48 \\
\rowcolor{blockred}  & $\alpha$-$\beta$ Div  & 1.37  & 2.05  & {10.17} & -1.26 & {3.08} \\
\rowcolor{blockred} \multirow{-5}{*}{\shortstack[l]{Gemma-2-2B-IT}} & \texttt{CaRE} & {3.6}   & {2.61}  & 5.79  & -0.4  & 2.90\\

\cdashline{1-7}

\rowcolor{blockblue}  & FKL & 0.18 & -0.31 & -0.96 & 0.09 & -0.25 \\
\rowcolor{blockblue}  & RKL & 0.02 & 0.18 & -0.29 & -0.03 & -0.04\\
\rowcolor{blockblue}  & SRKL & -0.15 & 0.67 & -1.48 & -0.10 & -0.26 \\
\rowcolor{blockblue}  & $\alpha$-$\beta$ Div  & 0.24 & {1.14} & {2.16} & 0.43 & {0.99} \\
\rowcolor{blockblue} \multirow{-5}{*}{\shortstack[l]{OLLaMA2-3B}} & \texttt{CaRE} & {0.27} & 0.26 & 1.82 & {0.57} & 0.73 \\

\bottomrule
\end{tabular}
}
\end{sc}
\end{small}
\end{table}

\begin{table*}[t]
\caption{\textbf{Generalization to specialized domains.} Cells report ROUGE-L / LLM-as-a-judge factuality. We evaluate three model families: chat alignment (Qwen2), code generation (DeepSeek-Coder), and math reasoning (Qwen2.5-Math). \loss\ consistently improves precision-sensitive tasks, achieving \textbf{+2.1 pp} on MBPP, \textbf{+1.7 pp} on GSM8k, and \textbf{+1.8 pp} on CollegeMath over the strongest baseline. Gains are statistically significant ($t=3.05$, $p=0.018$). Bold denotes the best score per cell. \rev{GKD and Distillm are additionally reported on the code and math tasks, where \method\ remains best on MBPP, GSM8K and CollegeMath and tied-best on HumanEval; chat-alignment entries for these two baselines are still running and omitted (``--'').}}
\label{tab:additional_results}
\centering
\begin{small}
\begin{sc}
\scalebox{0.85}{
\begin{tabular}{lccccccc}
\toprule
\multirow{3}{*}{\textbf{Method}} & \multicolumn{3}{c}{\cellcolor{blockblue}\textbf{\shortstack{Qwen2-7B-Inst \\ $\rightarrow$ Qwen2-1.5B}}} & \multicolumn{2}{c}{\cellcolor{blockred}\textbf{\shortstack{DS-Coder-6.9B-Inst \\ $\rightarrow$ DS-Coder-1.3B}}} & \multicolumn{2}{c}{\cellcolor{blockgreen}\textbf{\shortstack{Qwen2.5-Math-7B-Inst \\ $\rightarrow$ Qwen2.5-Math-1.5B}}} \\
 & \cellcolor{blockblue}Alpaca & \cellcolor{blockblue}Evol & \cellcolor{blockblue}Ultra & \cellcolor{blockred}HEval & \cellcolor{blockred}MBPP & \cellcolor{blockgreen}GSM8k &
 \cellcolor{blockgreen}College \\
 & \cellcolor{blockblue}\scriptsize{Rouge-L/LLM} & \cellcolor{blockblue}\scriptsize{Rouge-L/LLM} & \cellcolor{blockblue}\scriptsize{Rouge-L/LLM} & \cellcolor{blockred}\scriptsize{pass@1} & \cellcolor{blockred}\scriptsize{pass@1} & \cellcolor{blockgreen}\scriptsize{pass@1} &
 \cellcolor{blockgreen}\scriptsize{pass@1} \\
\midrule
Student & \cellcolor{blockblue}8.81/27.3 & \cellcolor{blockblue}14.9/31.2 & \cellcolor{blockblue}10.8/30.5 & \cellcolor{blockred}32.3 & \cellcolor{blockred}58.5 & \cellcolor{blockgreen}69.9 &
\cellcolor{blockgreen}37.1 \\
\cdashline{1-8}
Distillm2 & \cellcolor{blockblue}16.3/43.3 & \cellcolor{blockblue}26.5/52.1 & \cellcolor{blockblue}22.1/\textbf{49.3} & \cellcolor{blockred}\textbf{43.3} & \cellcolor{blockred}60.3 & \cellcolor{blockgreen}72.2 &
\cellcolor{blockgreen}44.2 \\
ABKD & \cellcolor{blockblue}\textbf{16.4}/\textbf{44.8} & \cellcolor{blockblue}26.7/51.9 & \cellcolor{blockblue}22.3/48.4 & \cellcolor{blockred}42.7 & \cellcolor{blockred}59.8 & \cellcolor{blockgreen}71.0 &
\cellcolor{blockgreen} 44.3 \\
\textcolor{black}{GKD} & \cellcolor{blockblue}\textcolor{black}{--} & \cellcolor{blockblue}\textcolor{black}{--} & \cellcolor{blockblue}\textcolor{black}{--} & \cellcolor{blockred}\textcolor{black}{43.1} & \cellcolor{blockred}\textcolor{black}{60.2} & \cellcolor{blockgreen}\textcolor{black}{71.4} & \cellcolor{blockgreen}\textcolor{black}{43.9} \\
\textcolor{black}{Distillm} & \cellcolor{blockblue}\textcolor{black}{--} & \cellcolor{blockblue}\textcolor{black}{--} & \cellcolor{blockblue}\textcolor{black}{--} & \cellcolor{blockred}\textcolor{black}{42.9} & \cellcolor{blockred}\textcolor{black}{60.0} & \cellcolor{blockgreen}\textcolor{black}{71.7} & \cellcolor{blockgreen}\textcolor{black}{44.1} \\
\cdashline{1-8}
\rowcolor{gray!10}
\method (Ours) & \textbf{16.4}/44.4 & \textbf{26.8/53.0} & \textbf{22.6}/\textbf{49.3} & \textbf{43.3} & \textbf{62.4} & \textbf{73.9} & \textbf{46.1} \\
\bottomrule
\end{tabular}
}
\end{sc}
\end{small}
\end{table*}

\subsection{Results on specialized domains}
\label{subsec:domain_results}

Table~\ref{tab:additional_results} extends evaluation to chat alignment (Qwen2-7B$\to$1.5B), code (DeepSeek-Coder-6.9B$\to$1.3B), and math (Qwen2.5-Math-7B$\to$1.5B). \loss\ leads on Evol-Instruct ($26.8$/$53.0$ vs.\ SRKL $26.5$/$52.1$) and UltraFeedback ($22.6$/$49.3$); reaches $62.4$ pass@1 on MBPP ($+2.1$ vs.\ SRKL, $+2.6$ vs.\ ABKD); and achieves $73.9\%$ on GSM8K ($+1.7$ vs.\ SRKL) and $46.1\%$ on CollegeMath ($+1.8$ vs.\ Distillm-2/ABKD). \rev{Against the additional on-policy and divergence baselines (GKD, Distillm) now reported in Table~\ref{tab:additional_results} for the code and math tasks, \method\ remains best on MBPP, GSM8K and CollegeMath and tied-best on HumanEval.} Precision-sensitive domains  ---  where a single hallucinated step can break a proof, a unit test, or a logical argument  ---  benefit most from confidence-conditioned mode seeking and rejection of unreliable teacher supervision.

\subsection{Hyperparameter sensitivity and ablations}
\label{subsec:ablation_analysis}


We conduct ablations on instruction-following tasks with both GPT- and OPT-based students (Figures~\ref{fig:ablation} and~\ref{fig:ablation_opt} in Appendix~\ref{app:detailed_results}). Several consistent trends emerge across architectures. Moderate-to-high skip percentages (roughly 50–70\%) consistently outperform low filtering regimes, indicating that a substantial portion of teacher supervision lies in epistemically unreliable regions. An increasing \revivalname\ schedule, where selectivity is gradually strengthened over training, reliably yields the best performance, supporting a curriculum effect in which early broad supervision transitions into progressive filtering as the student's confidence improves. Increasing the number of Monte Carlo samples used for BALD estimation improves stability, though strong performance is already achieved with as few as $N=3$ when combined with an increasing schedule. Hard gating attains the highest peak performance, while soft gating provides slightly greater robustness across configurations. Performance exhibits a unimodal dependence on sampling temperature, peaking at moderate values (around $T=3.0$), confirming that effective epistemic estimation requires controlled stochasticity rather than aggressive noise. We further ablate the margin $m$ in $g_{\mathrm{soft}}$ defined in Equation~\ref{eq:gates}. Lower $m$ performs the best, balancing the forward and reverse KL divergence metrics, while higher margins bias the objective and degrade performance. \rev{A small margin does not collapse the soft gate into hard gating: $m$ only shifts where the sigmoid is centered, not its slope, so at $m\!=\!0.05$ the gate is still a smooth, differentiable function of the confidence gap --- a near-symmetric switch centered near a gap of zero --- and the soft-gate selection term of Proposition~\ref{thm:grad_decomp} is retained. Its being the best margin is expected: it matches the hard-gate behavior of switching branches right at the crossover point, exactly where hard gating attains its highest peak, whereas a large $m$ biases the gate toward one geometry across most tokens. Soft gating is thus the robustness-oriented variant and hard gating the peak-performance variant.} The consistency of these trends across GPT and OPT families demonstrates that the design principles underlying \method\ generalize across architectures.

\begin{figure*}[!htb]
    \centering
    \includegraphics[width=\linewidth]{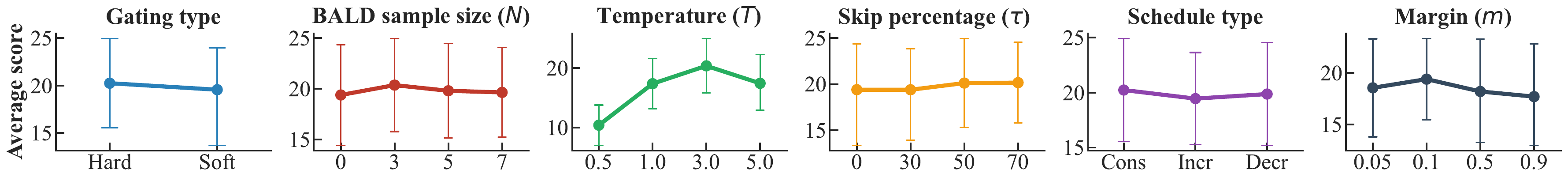}
    \caption{\textbf{Ablation study of \method\ with GPT2-base.} (a) gating strategy ($g_{\mathrm{hard}}$ vs.\ $g_{\mathrm{soft}}$); (b) MC-dropout sample count $N$; (c) sampling temperature $T$; (d) target skip percentage $\tau$; (e) \revivalname\ scheduling strategy (constant, increasing, decreasing); and (f) margin $m$ in soft gating.}
    \label{fig:ablation}
\end{figure*}

\section{Discussion}
\label{sec:discussion}

\noindent\textbf{Decoupling student confidence from teacher uncertainty.}
A core failure mode of standard KD is that the student inherits the teacher's epistemic uncertainty rather than its competence: when the teacher is unsure, Forward KL forces the student to spread probability mass over the teacher's heavy tail, eroding sharper, correct priors. Figure~\ref{fig:revival_bald} shows that \method\ explicitly avoids this trap. Across all four benchmarks, the teacher's BALD distribution exhibits a long, heavy tail of high-uncertainty samples; the student trained with \revivalname\ retains a tightly concentrated low-BALD profile that is closer to its pre-distillation prior than to the teacher's tail. The decoupling is large in magnitude: on Self-Instruct the teacher's mean BALD is $0.52$ while the corresponding distilled student reaches $0.077$, and the Kolmogorov--Smirnov distance between the two distributions exceeds $0.5$ on Vicuna and Dolly. \rev{Reading the pre- versus post-distillation KS drift (Figures~\ref{fig:motivation}a and~\ref{fig:revival_bald}) also isolates the two components of \method. Standard Skewed-RKL barely moves the student's uncertainty away from the teacher: the KS distance drifts by only $0.4\%$ ($0.840\!\to\!0.836$). The \loss\ token-level loss alone (\revivalname\ removed) already drifts by $2.4\%$ on Dolly ($0.840\!\to\!0.864$) --- several times the baseline --- and adding \revivalname\ raises the drift to $7\%$. The token-level loss is thus the main driver and \revivalname\ amplifies it; neither is a $0.4\%$ effect.} The same pattern is reproduced under static divergences augmented with \revivalname\ (Skewed RKL and $\alpha$--$\beta$, Figures~\ref{fig:revival_bald_distillm2}--\ref{fig:revival_bald_ab} in Appendix~\ref{app:discussion}), establishing that the effect is driven by selective rejection rather than by any one objective. Crucially, the rejection rate is not a static threshold: the proportion of batches triggering $M_{\mathrm{Revival}}$ rises monotonically through training (Figures~\ref{fig:conf_trend}--\ref{fig:conf_trend2}), reflecting a curriculum in which broad early supervision gives way to progressively stricter filtering as the student becomes more confident than the teacher on a growing subset of inputs. This empirically realizes the bi-directional conditional calibration formalized in Theorem~\ref{thm:conditional_calibration_overconfident}.

\begin{figure*}[!htb]
    \centering
    \includegraphics[width=\linewidth]{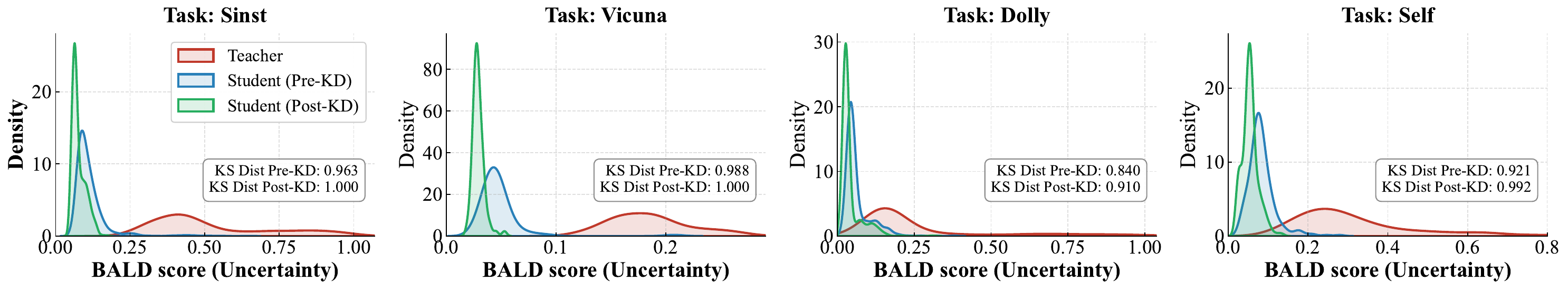}
    \caption{\textbf{BALD uncertainty distributions across four instruction-following benchmarks.} Teacher models show heavy high-uncertainty tails (red), whereas \method\ students (blue) suppress these regions and maintain concentrated low-BALD profiles. Reported KS distances indicate strong distributional separation, exceeding $0.5$ on Vicuna and Dolly.}
    \label{fig:revival_bald}
\end{figure*}

\noindent\textbf{Computational analysis.}
\begin{figure}[!htb]
    \centering
    \subfloat[Wall-clock vs.\ MC samples]{\includegraphics[width=0.46\linewidth]{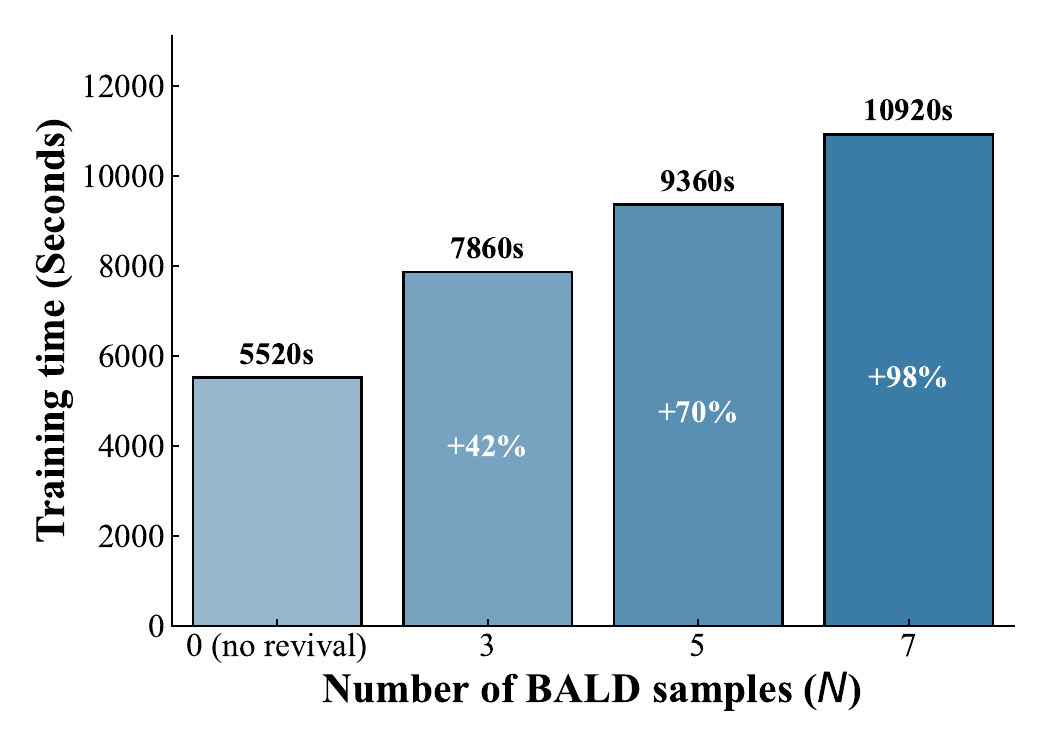}}\hfill
    \subfloat[Wall-clock vs.\ skip rate]{\includegraphics[width=0.46\linewidth]{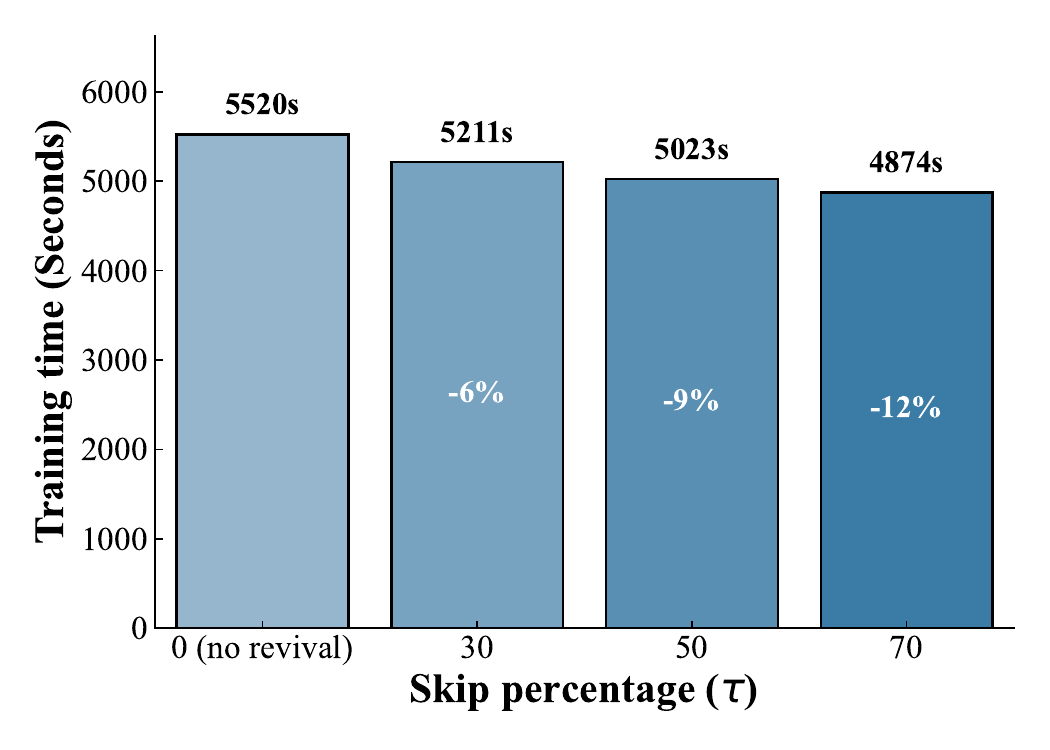}}
    \caption{\textbf{Runtime on Dolly-15k with OpenLLaMA-3B for one epoch.} (a) Wall-clock time scales near-linearly with the number of MC dropout samples $N$ for BALD estimation. (b) Higher skip rates $\tau$ reduce BALD overhead by skipping backward passes on rejected batches.}
    \label{fig:runtime}
\end{figure}
\revivalname\ trades a moderate amount of additional forward-pass compute for substantially more reliable supervision. The dominant cost is BALD estimation, which requires $N$ stochastic dropout passes per batch and scales near-linearly in $N$ (Figure~\ref{fig:runtime}(a)). Concretely, on Dolly with OpenLLaMA-3B, one epoch takes 5520s without \revivalname\ and 7860s with $N=3$, a $42\%$ overhead; pushing to $N=7$ approaches a $98\%$ overhead. This cost is partially recovered by the rejection mechanism itself: when $M_{\mathrm{Revival}}=1$, the corresponding backward pass is skipped, so increasing $\tau$ reduces effective wall-clock per epoch (Figure~\ref{fig:runtime}(b)); a $70\%$ skip target recovers roughly $12\%$ of the BALD overhead. \rev{This $\sim$$40\%$ figure is specific to instruction following, however; the per-domain overhead at $N=3$ is substantially larger, and is largest precisely when the training step is otherwise cheap (per-domain breakdown in Figure~\ref{fig:runtime2}, Appendix~\ref{app:discussion}): approximately $+106\%$ for chat alignment, $+170\%$ for code, and $+148\%$ for math. We therefore no longer describe the overhead as ``modest.'' \revivalname\ is a poor fit for settings with very large teachers, long sequences, or online/streaming updates; in those regimes the single-pass entropy proxy justified by Theorem~\ref{thm:conf_equiv} removes the extra stochastic passes at a small loss of precision. Because $N=3$ paired with an increasing schedule already achieves the strongest empirical performance and distillation is a one-time offline cost incurred before deployment, the overhead remains acceptable for the offline, high-compression settings \method\ targets.}

\noindent\textbf{Limitations.}
BALD-based rejection requires multiple stochastic passes and therefore costs more than single-pass distillation. As shown above, this overhead is $\sim$$40\%$ at $N=3$ for instruction following \rev{but considerably larger per domain (up to $\sim$$170\%$ on code), and is largest when the training step is otherwise cheap}; it is partially amortized by skipped updates, but in compute-constrained regimes, \rev{or with very large teachers, long sequences, or online updates}, it may still be undesirable. In such settings, the cheaper entropy proxy justified by Theorem~\ref{thm:conf_equiv} remains a valid drop-in for the rejection rule, with a small loss of precision. Second, the equivalence between entropy-based confidence and BALD relies on the Sharpness Hypothesis (Assumption~\ref{ass:sharpness}), which is best supported for cross-entropy-trained LLMs in regimes where individual MC samples are sharp. \rev{We stress that this assumption is used \emph{only} to justify the optional single-pass entropy fallback: every reported result uses the full MC-dropout BALD estimate, which does not rely on Assumption~\ref{ass:sharpness}.} In genuinely multimodal open-ended generation, where stochastic passes can disagree on \emph{which} valid continuation to predict rather than \emph{how confidently}, this assumption may weaken; relaxing the analysis to high-aleatoric regimes is a natural direction for future work.

\section{Conclusion}
\label{sec:conclu}

We introduced \method, a confidence-aware framework that redefines how supervision is transferred in generative knowledge distillation. By dynamically adapting the optimization geometry at the token level via \loss, the student selectively absorbs informative teacher signals while suppressing unreliable long-tail noise. Complementarily, the epistemic rejection mechanism \revivalname\ prevents the student from inheriting teacher hallucinations, preserving correct priors under uncertain supervision. Across instruction following, code generation, and mathematical reasoning, \method\ consistently yields students that are both more accurate and better calibrated than those trained with static divergence objectives. These findings underscore that effective distillation from modern LLMs requires adapting not only how much to imitate the teacher, but when to trust it.

\bibliographystyle{IEEEtranN}
\small
\bibliography{example_paper}

\beginappendix
\setcounter{tocdepth}{3}


\section{Related Work}
\label{sec:related}

\noindent \textbf{Knowledge distillation for LLMs.}
Knowledge Distillation (KD) has progressed from early logit-matching formulations \citep{hinton2015distilling} to methods explicitly designed for the generative and autoregressive nature of large language models. Initial extensions focused on sequence-level and white-box distillation objectives \citep{kim2016sequence, sun2019patient}, while more recent work addresses challenges such as exposure bias, vocabulary mismatch, and distributional shift during generation \citep{hsieh2023distilling, zhong2024revisiting}. Agarwal et al.~\citep{agarwal2024policy} proposed Generalized Knowledge Distillation (GKD), which performs on-policy distillation by training the student on its own generations scored by the teacher. Gu et al.~\citep{gu2024minillm} introduced MiniLLM, showing that Reverse KL can mitigate the zero-forcing behavior of Forward KL by reducing sensitivity to low-probability teacher tails. Sengupta et al.~\citep{sengupta2024a} explored collaborative distillation, where teacher and student roles evolve dynamically based on learning progress. Large-scale technical reports such as DeepSeek-V3 \citep{deepseekai2025deepseekr1incentivizingreasoningcapability} further demonstrate that distillation from strong reasoning models can substantially improve smaller models when teacher signals are carefully filtered. Complementing these findings, Ramesh et al.~\citep{ramesh-etal-2025-generalization} provide systematic evidence that noisy or inconsistent teacher supervision can collapse student performance. A growing body of work therefore emphasizes uncertainty-aware and selective distillation mechanisms to mitigate imperfect supervision \citep{fang2024bayesian, li2025bild, saadi2025taskdllm, he2025dakd}.

\noindent \textbf{Objective functions in knowledge distillation.}
The distillation objective determines the optimization geometry and, consequently, the behavior of the student. Most traditional approaches optimize Forward KL divergence \citep{hinton2015distilling, li2022asymmetric, zheng2024knowledge, sun2024logit}, which is mean-seeking and encourages the student to cover the teacher's full support. While this promotes recall, it also exposes the student to tail risk, allocating probability mass to incorrect or ambiguous tokens when the teacher distribution has high entropy \citep{gu2023knowledge}. In contrast, Reverse KL divergence \citep{kim2024promptkd, gu2024minillm} is mode-seeking and produces sharper predictions, but can induce mode collapse and reduced diversity. To balance these extremes, recent methods propose static interpolations. Skewed KL \citep{ko2024distillm} reweights gradients through a mixture target to stabilize training, while $\alpha$--$\beta$ divergence \citep{wangabkd} provides a parametric family that controls tail weighting via fixed hyperparameters. Other explored alternatives include Pearson correlation \citep{huang2022knowledge}, total variation distance \citep{wen2023f}, Rényi divergence \citep{renyi_2023}, and Wasserstein distance \citep{lv2024wasserstein}. Despite their expressiveness, these objectives impose a global geometric preference across all tokens and do not adapt to token-level variability in teacher reliability.

In contrast, the confidence-gated objective \loss\ operates as a dynamic geometric switch rather than a static interpolation. By selecting Forward KL when the teacher is confident and Reverse KL when the teacher is uncertain, it enables token-level adaptation of the precision--recall trade-off \citep{guo2025does}. Moreover, by explicitly incorporating both aleatoric and epistemic uncertainty through the \revivalname\ mechanism, our approach selectively rejects unreliable teacher supervision when the student is already confident. This combination of adaptive geometry and selective rejection distinguishes our method from prior divergence-based and uncertainty-aware distillation approaches.

\section{Algorithmic Details}

\subsection{Estimation of Epistemic Uncertainty (BALD)}
\label{subsec:bald_calc}

We estimate epistemic uncertainty using BALD computed via Monte Carlo (MC) dropout. For a batch of inputs, we perform $S$ stochastic forward passes, producing logits $\mathbf{Z}\in\mathbb{R}^{S\times B\times L\times V}$, where $B$ is batch size, $L$ is sequence length, and $V$ is vocabulary size. \rev{We do not depend on the base model's deployment-time dropout setting: for the BALD passes we force all dropout layers into training mode and set attention dropout to $p=0.1$, then restore the original configuration afterward. This guarantees a non-degenerate stochastic signal even for models shipped with dropout disabled (a verification on both teachers used here is given in Appendix~\ref{app:uncertainty_verification}).} Let
\begin{equation}
\mathbf{P}_{s,b,t}=\mathrm{softmax}(\mathbf{Z}_{s,b,t})\in\Delta^{|V|-1}
\end{equation}
denote the corresponding next-token distributions. The posterior predictive distribution is approximated by the mean probability
\begin{equation}
\bar{\mathbf{P}}=\frac{1}{S}\sum_{s=1}^{S}\mathbf{P}_s\in\mathbb{R}^{B\times L\times |V|}.
\end{equation}
BALD is the mutual information between predictions and model parameters, which admits the standard decomposition into total predictive uncertainty and expected (aleatoric) uncertainty. Let $H(\cdot)$ denote Shannon entropy over the vocabulary dimension:
\begin{align}
\mathbf{H}_{\mathrm{total}} &= H(\bar{\mathbf{P}})\in\mathbb{R}^{B\times L},\\
\mathbf{H}_{\mathrm{aleatoric}} &= \frac{1}{S}\sum_{s=1}^{S}H(\mathbf{P}_s)\in\mathbb{R}^{B\times L}.
\end{align}
The token-wise epistemic uncertainty is then
\begin{equation}
\mathbf{I}_{b,t}=\mathbf{H}_{\mathrm{total},b,t}-\mathbf{H}_{\mathrm{aleatoric},b,t}\in\mathbb{R}^{B\times L}.
\end{equation}
To obtain a sequence-level signal for the rejection mask, we mean-pool across tokens:
\begin{equation}
\mathcal{U}_{\mathrm{BALD}}^{(b)}=\frac{1}{L}\sum_{t=1}^{L}\mathbf{I}_{b,t}.
\end{equation}

{\color{black}
\subsection{Verification that the uncertainty estimator is non-degenerate}
\label{app:uncertainty_verification}

MC-dropout uncertainty is known to be coarse \citep{gal2016dropout, ovadia2019trust}, and a natural concern is that it could collapse to a constant for models not trained with meaningful dropout. Because \revivalname\ forces dropout to $p=0.1$ for the estimation passes (Appendix~\ref{subsec:bald_calc}), the signal remains informative. We verify this directly on the two teachers used in our instruction-following experiments. On 500 Dolly evaluation sequences with $N=3$ passes, every sequence has non-zero BALD and the top-1 probability moves appreciably across passes, confirming a genuine stochastic signal rather than a degenerate estimate.

\begin{table}[h]
\centering
\color{black}
\caption{Verification that MC-dropout BALD is non-degenerate on the two teachers (500 Dolly sequences, $N=3$, dropout forced to $p=0.1$).}
\label{tab:bald_verification}
\begin{tabular}{lccccc}
\toprule
\textbf{Teacher} & \textbf{$N$} & \textbf{Mean BALD} & \textbf{BALD range} & \textbf{Non-zero BALD} & \textbf{Mean $\Delta$ top-1 prob} \\
\midrule
GPT2-XL & 3 & 0.049 & 0.013--0.107 & 500/500 (100\%) & 0.074 \\
OPT-2.7B & 3 & 0.152 & 0.057--0.239 & 500/500 (100\%) & 0.079 \\
\bottomrule
\end{tabular}
\end{table}

\noindent Crucially, \revivalname\ uses only a \emph{relative} teacher-vs-student comparison via running quantiles rather than a calibrated absolute value, which tolerates a coarse estimate; this is why performance is flat between $N=3$ and $N=7$ (Section~\ref{subsec:ablation_analysis}).

\subsection{Numerical stability of the divergence terms}
\label{app:numerical_stability}

A potential concern is how $\mathrm{KL}(q\|p)$ is kept finite when the teacher assigns near-zero probability to some tokens. The teacher is a tempered softmax, so $p$ has full support by construction. In practice, log-probabilities are taken directly from \texttt{log\_softmax} and clamped to a finite floor before use, and the $\alpha$--$\beta$ objective additionally mixes in a small uniform mass before taking the logarithm. Consequently no \texttt{inf} or \texttt{NaN} values occurred in any of our runs.
}

\section{Proofs of Theoretical Results}
\label{app:proofs}

\subsection{Gradient of \loss}
\label{app:proof_grad}

\noindent \textbf{Scope and assumptions.}
This decomposition applies when the gate $g$ is differentiable and not detached from the computation graph (e.g., soft gating). If $g$ is detached, then $\nabla_\theta g \equiv 0$ and the selection term below vanishes.

\begin{proof}
Let $\theta$ parameterize the student distribution $q_\theta$ and define
\[
\mathcal{L}_{\mathrm{gKL}}(\theta)=g(\theta)\,\mathcal{L}_{\mathrm{FKL}}(\theta)+(1-g(\theta))\,\mathcal{L}_{\mathrm{RKL}}(\theta),
\]
where $\mathcal{L}_{\mathrm{FKL}}(\theta)=\mathrm{KL}(p\|q_\theta)$ and $\mathcal{L}_{\mathrm{RKL}}(\theta)=\mathrm{KL}(q_\theta\|p)$, and $p$ is treated as fixed with respect to $\theta$. Applying the product rule,
\begin{align}
\nabla_\theta \mathcal{L}_{\mathrm{gKL}}
&=\nabla_\theta\!\big(g\,\mathcal{L}_{\mathrm{FKL}}\big)+\nabla_\theta\!\big((1-g)\,\mathcal{L}_{\mathrm{RKL}}\big)\\
&=g\,\nabla_\theta \mathcal{L}_{\mathrm{FKL}}+\mathcal{L}_{\mathrm{FKL}}\,\nabla_\theta g+(1-g)\,\nabla_\theta \mathcal{L}_{\mathrm{RKL}}-\mathcal{L}_{\mathrm{RKL}}\,\nabla_\theta g\\
&=\Big[g\,\nabla_\theta \mathcal{L}_{\mathrm{FKL}}+(1-g)\,\nabla_\theta \mathcal{L}_{\mathrm{RKL}}\Big]+\Big[(\mathcal{L}_{\mathrm{FKL}}-\mathcal{L}_{\mathrm{RKL}})\,\nabla_\theta g\Big],
\end{align}
which yields the learning term and the gate-induced term.

For soft gating, let
\[
g=\sigma\!\big(c(p)-c(q_\theta)-m\big),\qquad \Delta \triangleq c(p)-c(q_\theta)-m.
\]
Then $\nabla_\theta g = \sigma'(\Delta)\,\,\nabla_\theta \Delta = g(1-g)\,\,\nabla_\theta\Delta$ and since $c(p)$ and $m$ are constants w.r.t.\ $\theta$,
\[
\nabla_\theta \Delta = -\nabla_\theta c(q_\theta).
\]
Therefore,
\begin{equation}
\label{eq:grad_gate_correct}
\nabla_\theta g = -\,g(1-g)\,\nabla_\theta c(q_\theta),
\end{equation}
and the gate-induced term becomes
\begin{equation}
\label{eq:selection_signal_correct}
(\mathcal{L}_{\mathrm{FKL}}-\mathcal{L}_{\mathrm{RKL}})\,\nabla_\theta g
=
-\,g(1-g)\,(\mathcal{L}_{\mathrm{FKL}}-\mathcal{L}_{\mathrm{RKL}})\,\nabla_\theta c(q_\theta).
\end{equation}
This concludes the derivation.
\end{proof}

\subsection{Gradient of Skewed Reverse KL}
\label{app:srkl_grad}

We derive the gradient of Skewed Reverse KL, $\mathrm{KL}(q\|m)$, with $m=\lambda p+(1-\lambda)q$ and $\lambda\in[0,1]$, with respect to student logits $z$ where $q=\mathrm{softmax}(z)$.

\begin{proof}
Define
\[
\mathcal{L}=\mathrm{KL}(q\|m)=\sum_k q_k\log\frac{q_k}{m_k}
=\sum_k q_k\log q_k-\sum_k q_k\log m_k.
\]
First compute $\partial \mathcal{L}/\partial q_i$. For the entropy term,
\[
\frac{\partial}{\partial q_i}\sum_k q_k\log q_k = 1+\log q_i.
\]
For the cross term, apply the product rule:
\[
\frac{\partial}{\partial q_i}\sum_k q_k\log m_k = \log m_i+\sum_k q_k\,\frac{\partial \log m_k}{\partial q_i}.
\]
Since $m_k=\lambda p_k+(1-\lambda)q_k$, we have $\frac{\partial m_k}{\partial q_i}=(1-\lambda)\delta_{ki}$, hence
\[
\frac{\partial \log m_k}{\partial q_i}=\frac{1}{m_k}(1-\lambda)\delta_{ki},
\qquad
\sum_k q_k\,\frac{\partial \log m_k}{\partial q_i}=(1-\lambda)\frac{q_i}{m_i}.
\]
Therefore,
\begin{equation}
\label{eq:dL_dq_srkl}
\frac{\partial \mathcal{L}}{\partial q_i}
=(1+\log q_i)-\left(\log m_i+(1-\lambda)\frac{q_i}{m_i}\right)
=\log\frac{q_i}{m_i}+1-(1-\lambda)\frac{q_i}{m_i}.
\end{equation}

Next, use the softmax Jacobian $\frac{\partial q_k}{\partial z_j}=q_k(\delta_{kj}-q_j)$:
\begin{equation}
\label{eq:dL_dz_srkl_centered}
\frac{\partial \mathcal{L}}{\partial z_j}
=\sum_k \frac{\partial \mathcal{L}}{\partial q_k}\frac{\partial q_k}{\partial z_j}
=q_j\left(v_j-\sum_k q_k v_k\right),
\quad
v_k \triangleq \frac{\partial \mathcal{L}}{\partial q_k}.
\end{equation}
Substituting \eqref{eq:dL_dq_srkl} into $v_k$, note that the additive constant ``$+1$'' cancels inside the centering term in \eqref{eq:dL_dz_srkl_centered}. Thus an equivalent vector form is
\begin{equation}
\label{eq:srkl_grad_vector}
\nabla_z \mathcal{L}
=
q \odot \left(\log\frac{q}{m}-(1-\lambda)\frac{q}{m}-\mathbb{E}_{y\sim q}\!\left[\log\frac{q_y}{m_y}-(1-\lambda)\frac{q_y}{m_y}\right]\mathbf{1}\right),
\end{equation}
where $\mathbf{1}$ is the all-ones vector and the final term is the centering constant.
\end{proof}

\subsection{Worst-case robustness under \revivalname}
\label{app:worst_case_proof}

\begin{proof}
When $M_{\mathrm{Revival}}=1$, the per-batch objective satisfies $\mathcal{L}_{\mathrm{CaRE\text{-}KD}} \equiv 0$ by construction (Eq.~\ref{eq:final_loss}). Hence $\nabla_\theta \mathcal{L}_{\mathrm{CaRE\text{-}KD}} = 0$ and the student parameters are unchanged at this batch. Therefore for any reference distribution $p^\star$, $\mathrm{KL}(p^\star \| q_{\mathrm{new}})=\mathrm{KL}(p^\star \| q_{\mathrm{old}})$, proving the student does not degrade due to noisy teacher supervision in the rejected regime.
\end{proof}

\subsection{Proof of Proposition~\ref{thm:skewed_rkl_relation} (Limiting equivalence to SRKL)}
\label{app:proof_thm32}

\begin{proof}
From \eqref{eq:srkl_grad_vector}, as $\lambda\to 1$ we have $m=\lambda p+(1-\lambda)q\to p$ and $(1-\lambda)\frac{q}{m}\to 0$. Therefore,
\[
\nabla_z \mathrm{KL}(q\|m)\ \longrightarrow\ \nabla_z \mathrm{KL}(q\|p),
\]
i.e., the Skewed Reverse KL gradient converges to the standard Reverse KL gradient. Under \loss, in the regime $g=0$ the student optimizes $\mathrm{KL}(q\|p)$ per token, matching this limiting geometry pointwise. SRKL applies a fixed $\lambda$ to every token, so its global gradient is $\sum_t \nabla_z \mathrm{KL}(q_t\|m_t)$, which cannot equal $\sum_t [g_t \nabla_z \mathrm{KL}(p_t\|q_t)+(1-g_t)\nabla_z \mathrm{KL}(q_t\|p_t)]$ for general token-varying $g_t\in\{0,1\}$ unless $g_t$ is constant across tokens. This proves the strict separation of \loss\ from any global SRKL.
\end{proof}

\subsection{Proof of Theorem~\ref{thm:conf_equiv} (Conditional equivalence of confidence proxies)}
\label{app:proof_thm45}

\noindent \textbf{Clarification.}
The claim is conditional: under the Sharpness Hypothesis and the additional approximation that a single forward pass $\hat{p}$ is close to the posterior predictive mean $p(y\mid\mathbf{x})=\mathbb{E}_\omega[p(y\mid\mathbf{x},\omega)]$, minimizing single-pass entropy serves as a proxy for minimizing BALD.

\begin{proof}
Using the standard BALD decomposition,
\begin{equation}
\label{eq:bald_decomposition}
\mathcal{I}_{\mathrm{BALD}}(\mathbf{x})
=
H\!\left[\mathbb{E}_{\omega}p(y\mid\mathbf{x},\omega)\right]
-
\mathbb{E}_{\omega}\!\left[H\!\left(p(y\mid\mathbf{x},\omega)\right)\right].
\end{equation}
Under the Sharpness Hypothesis, $H[p(y\mid\mathbf{x},\omega)]\approx 0$ for $\omega$ in the posterior support, hence the second term is approximately zero and
\begin{equation}
\label{eq:bald_entropy_proxy}
\mathcal{I}_{\mathrm{BALD}}(\mathbf{x}) \approx H\!\left[p(y\mid\mathbf{x})\right],
\quad
p(y\mid\mathbf{x})\triangleq \mathbb{E}_{\omega}p(y\mid\mathbf{x},\omega).
\end{equation}
The normalized entropy confidence is
\[
c(\hat{p})=1-\frac{H(\hat{p})}{\log|\mathcal{V}|}.
\]
Since $\log|\mathcal{V}|>0$ is constant, maximizing $c(\hat{p})$ is equivalent to minimizing $H(\hat{p})$. If the single-pass distribution $\hat{p}$ is a reasonable approximation to the mean predictive distribution $p(\cdot\mid\mathbf{x})$, then minimizing $H(\hat{p})$ approximately minimizes $H[p(\cdot\mid\mathbf{x})]$. By \eqref{eq:bald_entropy_proxy}, this approximately minimizes $\mathcal{I}_{\mathrm{BALD}}(\mathbf{x})$. Therefore, under these assumptions, normalized-entropy confidence can be used as a computationally efficient proxy for BALD in the inner loop.
\end{proof}

\subsection{Proof of Theorem~\ref{thm:conditional_calibration_overconfident} (Conditional calibration of overconfident students)}
\label{app:proof_thm46_2}

\begin{proof}
Consider a fixed branch regime of $\mathcal{L}_{\mathrm{CaRE\text{-}KD}}$ in
which the branch selection does not change along the optimization trajectory.
Within such a regime, the objective reduces to one of the KL branches. In the
Forward KL branch, the objective is
\[
\min_q \mathrm{KL}(p\|q),
\]
whose unique global minimizer is $q=p$, assuming $q$ has support wherever $p$
has support. In the Reverse KL branch, the objective is
\[
\min_q \mathrm{KL}(q\|p),
\]
which is also minimized at $q=p$, assuming $p$ has full support. Hence, in either
fixed branch regime considered here, convergence to the branch minimizer implies
\[
q_{\theta_t} \to p.
\]
By continuity of entropy on the probability simplex under the stated support
conditions, it follows that
\[
H(q_{\theta_t}) \to H(p).
\]

Since the student is initialized overconfident relative to the teacher,
\[
H(q_{\theta_0}) < H(p).
\]
The limiting entropy is $H(p)$, so the student entropy must increase at some
point along the trajectory in order to reach this limit. Therefore,
$\mathcal{L}_{\mathrm{CaRE\text{-}KD}}$ acts, in this fixed branch regime, as a
conditional calibration mechanism that moves an overconfident student toward
the teacher's entropy level.
\end{proof}

\subsection{Proof of Theorem~\ref{thm:conditional_sharpening_underconfident} (Conditional sharpening of underconfident students)}
\label{app:proof_thm47_2}

\begin{proof}
Consider a fixed branch regime of $\mathcal{L}_{\mathrm{CaRE\text{-}KD}}$ in
which the branch selection does not change along the optimization trajectory.
Within such a regime, the objective reduces to one of the KL branches. In the
Forward KL branch, the objective is
\[
\min_q \mathrm{KL}(p\|q),
\]
whose unique global minimizer is $q=p$, assuming $q$ has support wherever $p$
has support. In the Reverse KL branch, the objective is
\[
\min_q \mathrm{KL}(q\|p),
\]
which is also minimized at $q=p$, assuming $p$ has full support. Hence, in either
fixed branch regime considered here, convergence to the branch minimizer implies
\[
q_{\theta_t} \to p.
\]
By continuity of entropy on the probability simplex under the stated support
conditions, it follows that
\[
H(q_{\theta_t}) \to H(p).
\]

Since the student is initialized underconfident relative to the teacher,
\[
H(q_{\theta_0}) > H(p).
\]
The limiting entropy is $H(p)$, so the student entropy must decrease at some
point along the trajectory in order to reach this limit. Therefore,
$\mathcal{L}_{\mathrm{CaRE\text{-}KD}}$ acts, in this fixed branch regime, as a
conditional sharpening mechanism that moves an underconfident student toward
the teacher's entropy level.
\end{proof}

\section{Experimental Details}
\label{app:exp_details}

\subsection{Training Datasets}
\label{app:training_data}

We train student models using a diverse collection of instruction-tuning and domain-specific datasets, following standard practices in recent large-scale distillation studies \citep{ko2024distillm, wangabkd}. The training corpus spans general instruction following, chat-style alignment, code generation, and mathematical reasoning.

\noindent \textbf{Dolly-15k} \citep{conover2023free} is a curated instruction-following dataset containing approximately 15k single-turn instruction--response pairs across tasks such as open-ended question answering, summarization, classification, information extraction, and creative generation. We use Dolly-15k as our primary dataset for controlled instruction-following distillation, consistent with prior work analyzing robustness under noisy teacher supervision.

\noindent \textbf{UltraChat-200k} \citep{ding2023enhancing} is a large-scale conversational dataset designed for alignment-oriented training. It consists of multi-turn dialogues generated by strong LLMs and emphasizes helpfulness, harmlessness, and coherence. We use UltraChat-200k to distill chat-optimized student models from larger teachers, following recent alignment-focused distillation setups.

\noindent \textbf{WizardCoder} \citep{luo2024wizardcoder} is a code generation dataset constructed using Evol-Instruct, comprising programming problems paired with step-by-step solutions. The dataset emphasizes logical correctness across varying difficulty levels and is used exclusively to train code-specialized student models without additional filtering or reformatting.

\noindent \textbf{MetaMathQA} \citep{yumetamath} is a large-scale mathematical reasoning dataset containing math word problems with structured reasoning traces. It spans arithmetic, algebra, and multi-step symbolic reasoning and is used to train math-specialized student models under standard zero-shot evaluation protocols.

\subsection{Evaluation Benchmarks}
\label{app:evaluation_data}

We evaluate distilled models on a diverse suite of held-out benchmarks spanning instruction following, alignment, code generation, and mathematical reasoning.

\noindent \textbf{Dolly Evaluation} \citep{conover2023free} consists of held-out instruction-following prompts covering summarization, question answering, classification, and open-ended generation.

\noindent \textbf{Self-Instruct} \citep{wang2023self} contains automatically generated instructions with diverse task formulations, designed to test robustness and generalization beyond curated templates.

\noindent \textbf{Super-Natural Instructions} \citep{wang-etal-2022-super} is a heterogeneous benchmark spanning hundreds of task types, evaluating generalization across unseen instructions and domains.

\noindent \textbf{Vicuna Instructions} \citep{vicuna2023} evaluates conversational response quality on user-style prompts, emphasizing helpfulness and coherence.

\noindent \textbf{AlpacaEval} \citep{li2023alpacaeval} is a pairwise preference benchmark assessing alignment quality via automated and human-aligned judges.

\noindent \textbf{Evol-Instruct} \citep{xu2024wizardlm} evaluates a model's ability to handle increasingly complex and compositional instructions generated through iterative evolution.

\noindent \textbf{UltraFeedback} \citep{cuiultrafeedback} evaluates alignment using preference-based feedback signals, focusing on helpfulness, safety, and response quality.

\noindent \textbf{HumanEval} \citep{chen2021evaluatinglargelanguagemodels} measures code generation performance using execution-based correctness on Python programming problems.

\noindent \textbf{MBPP} \citep{austin2021program} evaluates code synthesis on basic programming tasks using unit-test-based validation.

\noindent \textbf{GSM8k} \citep{cobbe2021training} evaluates multi-step mathematical reasoning using grade-school word problems, reporting exact-match accuracy.

\noindent \textbf{CollegeMath} \citep{tang2024mathscale} consists of college-level problems drawn from undergraduate textbooks, covering algebra, calculus, linear algebra, and differential equations.

\subsection{Training Details}
\label{app:task_training}

\noindent \textbf{Instruction Following.}
For Dolly-15k, we follow the setup of \citet{gu2024minillm}. For student models with fewer than 1B parameters, we use a learning rate of $5\times10^{-5}$, batch size 32, and train for up to 20 epochs. For models larger than 1B parameters, we use the same learning rate with a reduced batch size of 8 and train for 10 epochs. All instruction-following models are trained using LoRA \citep{hu2021lora} adapters applied to all linear layers in the self-attention and MLP blocks, with rank $r=16$. Model checkpoints are selected based on validation ROUGE-L.

For UltraChat-200k, WizardCoder, and MetaMathQA, student models are trained with a learning rate of $5\times10^{-5}$, LoRA rank $r=16$, and trained for 3, 2, and 2 epochs respectively.

\noindent \textbf{Hyperparameters for \method.}
Epistemic uncertainty is estimated using Monte Carlo dropout with $N \in \{3,5,7\}$ stochastic forward passes. Unless otherwise stated, results are reported with $N=3$, which provides a favorable trade-off between computational cost and estimation fidelity. The \revivalname\ mechanism is controlled via a target skip percentage $\tau$, evaluated under constant, increasing, and decreasing schedules. The sampling temperature for stochastic forward passes is selected from $\{0.5, 1.0, 3.0, 5.0\}$, with moderate temperatures (default $3.0$) yielding the most reliable uncertainty estimates.

\noindent \textbf{Scheduling functions.}
We employ continuous scheduling functions to control the target skip percentage $\tau$ over training iterations. Let $e \in [0,T]$ denote the current epoch, where $T>0$ is the total scheduling horizon.

\textbf{Exponential increasing schedule.}
\begin{equation}
\label{eq:exp_increasing}
s_{\mathrm{inc}}(e)
=
\tau \cdot
\frac{1 - \exp\!\left(-k \frac{e}{T}\right)}
     {1 - \exp(-k)},
\qquad k > 0 .
\end{equation}
This schedule gradually increases selectivity, enabling early broad supervision followed by progressively stricter rejection as the student becomes more confident.

\textbf{Exponential decreasing schedule.}
\begin{equation}
\label{eq:exp_decreasing}
s_{\mathrm{dec}}(e)
=
\tau \cdot
\frac{\exp\!\left(-a \frac{e}{T}\right) - \exp(-a)}
     {1 - \exp(-a)},
\qquad a > 0 .
\end{equation}
This schedule relaxes selectivity over time. In the limit $a \to 0$, it reduces to linear decay:
\begin{equation}
\label{eq:linear_limit}
\lim_{a \to 0} s_{\mathrm{dec}}(e) = \tau \left(1 - \frac{e}{T}\right).
\end{equation}

\subsection{Evaluation Protocol}
\label{app:evaluation_protocol}

Instruction-following responses are sampled using temperature 0.8, top-$p$ 0.95, and a maximum length of 512 tokens. AlpacaEval uses \texttt{text-davinci-003} references, while Evol-Instruct and UltraFeedback use \texttt{gpt-3.5-turbo}. Results are averaged over five random seeds.

For mathematical reasoning and code generation, greedy decoding with a maximum length of 1024 tokens is used. Evaluation is performed using LM-Evaluation-Harness \citep{eval-harness} and EvalPlus \citep{liu2023your}.

\subsection{Prompt Structure}
\label{app:prompt_template}

We adopt a unified prompt template following \citet{gu2024minillm}. During student-generated-output (SGO) training, responses are sampled from the student; otherwise, ground-truth responses are used. \rev{In the SGO setting the teacher scores the student's own generations, and the confidence gate of Equation~\ref{eq:gates} is computed on those same on-policy tokens, so the distillation loss and the gate always share identical contexts; in the non-SGO setting both are computed on the teacher-forced tokens.}

\subsection{Hardware}

All experiments are conducted on a single NVIDIA A100 80GB GPU. Evaluation generations use vLLM \citep{kwon2023efficient}.

\subsection{LLM-as-a-judge Factuality Evaluation Prompt}
\label{app:llm_judge_prompt}

For each (instruction, reference, candidate) triple, we query GPT-5-Mini in a stateless single-turn setting using the following system prompt. The judge returns a single integer factuality score in $[0, 100]$, which we then average across the evaluation set. To minimize position bias, we present the reference and candidate in a fixed order; to minimize stochastic variance, we use temperature $0$ and average over three independent calls per example.

\begin{quote}\small\ttfamily
You are an impartial grader evaluating the factual correctness of a model's response against a trusted reference answer. You will be given (i) an instruction, (ii) a reference answer, and (iii) a candidate response.

Score the candidate from 0 to 100 according to the following rubric:
\begin{itemize}[leftmargin=*, itemsep=0pt, parsep=0pt, topsep=0pt]
\item 90--100: All factual claims in the candidate are correct and consistent with the reference. No fabricated entities, numbers, dates, or relationships.
\item 70--89: The candidate is mostly correct but contains minor inaccuracies (e.g., off-by-one numerical values, slightly imprecise paraphrase) that do not change the core meaning.
\item 40--69: The candidate contains at least one substantive factual error or unsupported claim that materially changes the answer.
\item 10--39: The candidate is largely incorrect, with multiple hallucinated facts, fabricated citations, or contradictions of the reference.
\item 0--9: The candidate is irrelevant, refuses the task, or is entirely fabricated.
\end{itemize}

Ignore stylistic differences (length, tone, formatting). Penalize only factual deviations.

Return only a single integer in [0, 100]. Do not produce any other text, justification, or formatting.
\end{quote}

This prompt is held fixed across every model, dataset, and decoding regime to ensure that LLM-judge scores are comparable across cells of Tables~\ref{tab:main_results}, \ref{tab:divergence_comparison}, and~\ref{tab:additional_results}.

\section{Additional Results}
\label{app:detailed_results}

\begin{table*}[t]
\caption{\textbf{Comparison of divergence measures without student-generated outputs (non-SGO)}~\citep{ko2024distillm}. Each cell reports ROUGE-L / LLM-as-a-judge factuality scores. Even without exploratory rollouts, \loss\ is competitive with the strongest baseline on every model pair, achieving the highest average ROUGE-L on three of five teacher--student pairs and matching SRKL on the remaining two within standard error. Improvements over FKL/RKL are statistically significant in a paired $t$-test ($t=12.84$, $p<0.005$). Bold marks the best score per (model, task, metric) cell.}
\label{tab:divergence_comparison}
\centering
\begin{small}
\begin{sc}
\scalebox{0.88}{
\begin{tabular}{llccccc}
\toprule
\multirow{2}{*}{\textbf{Teacher $\rightarrow$ Student}} & \multirow{2}{*}{\textbf{Loss}} & \multicolumn{4}{c}{\textbf{Benchmarks}} & \multirow{2}{*}{\textbf{Avg}} \\
\cmidrule(lr){3-6}
  & & Dolly & Self-Inst & Super-Nat & Vicuna & \\
\midrule

\rowcolor{blockblue}  & FKL & 23.0/20.4 & 9.7/14.2 & 16.8/14.9 & 14.9/13.0 & 16.1/15.6 \\
\rowcolor{blockblue}  & RKL & 23.0/20.1 & 10.1/14.9 & 17.7/15.3 & 15.5/13.7 & 16.6/16.0 \\
\rowcolor{blockblue}  & SRKL & 23.5/21.7 & 11.1/16.2 & 21.5/18.8 & 14.8/14.2 & 17.7/17.7 \\
\rowcolor{blockblue}  & $\alpha$-$\beta$ Div & 23.0/21.6 & 10.3/15.3 & 19.7/17.0 & 14.3/13.8 & 16.8/16.9 \\
\rowcolor{blockblue} \multirow{-5}{*}{\shortstack[l]{GPT2-XL (1.5B) $\rightarrow$ \\ GPT2-base (0.1B)}} 
& \texttt{CaRE-Div} & \textbf{25.2/23.5} & \textbf{12.5/18.6} & \textbf{23.7/21.8} & \textbf{16.3/14.8} & \textbf{19.5/19.7} \\

\cdashline{1-7}

\rowcolor{blockred}  & FKL & 22.5/19.5 & 10.1/13.5 & 14.8/12.9 & 14.4/12.8 & 15.4/14.7 \\
\rowcolor{blockred}  & RKL & 18.2/13.2 & 7.8/12.6 & 11.0/10.8 & 13.8/11.5 & 12.7/12.0 \\
\rowcolor{blockred}  & SRKL & 22.1/19.1 & 10.3/13.8 & 19.0/16.3 & 14.0/12.9 & 16.3/15.5 \\
\rowcolor{blockred}  & $\alpha$-$\beta$ Div & 21.6/18.5 & 8.6/12.9 & 15.9/14.2 & 14.6/13.0 & 15.2/14.7 \\
\rowcolor{blockred} \multirow{-5}{*}{\shortstack[l]{GPT2-XL (1.5B) $\rightarrow$ \\ GPT2-large (0.8B)}} 
& \texttt{CaRE-Div} & \textbf{25.0/24.1} & \textbf{11.5/15.6} & \textbf{19.6/16.9} & \textbf{15.4/13.7} & \textbf{17.9/17.6} \\

\cdashline{1-7}

\rowcolor{blockorange}  & FKL & 22.0/19.8 & 9.4/13.8 & 17.0/14.7 & 15.3/13.3 & 15.9/15.4 \\
\rowcolor{blockorange}  & RKL & 25.1/23.0 & 11.5/15.7 & 22.3/18.5 & \textbf{17.0/16.1} & 19.0/18.3 \\
\rowcolor{blockorange}  & SRKL & \textbf{25.6/23.7} & \textbf{12.3/18.9} & \textbf{23.6/20.4} & 15.9/15.1 & \textbf{19.4/19.5} \\
\rowcolor{blockorange}  & $\alpha$-$\beta$ Div & 24.5/22.8 & 12.1/17.3 & 21.5/18.0 & 16.7/14.0 & 18.7/18.0 \\
\rowcolor{blockorange} \multirow{-5}{*}{\shortstack[l]{OPT-2.7B $\rightarrow$ \\ OPT-125M}} 
& \texttt{CaRE-Div} & 23.9/20.6 & 11.3/17.0 & 22.3/19.2 & 16.4/13.7 & 18.5/17.6 \\

\cdashline{1-7}

\rowcolor{blockred}  & FKL & 17.7/18.9 & \textbf{11.0/21.3} & 18.3/14.7 & 16.5/15.4 & 15.9/17.6 \\
\rowcolor{blockred}  & RKL & 16.9/19.5 & 9.2/21.0 & \textbf{18.9}/18.4 & 19.2/\textbf{21.0} & 16.0/20.0 \\
\rowcolor{blockred}  & SRKL & \textbf{18.8}/23.6 & 9.9/20.2 & 15.6/17.5 & 17.8/17.3 & 15.5/19.7 \\
\rowcolor{blockred}  & $\alpha$-$\beta$ Div & 18.1/\textbf{24.2} & 9.8/20.9 & 16.7/16.5 & 19.1/19.5 & 15.9/20.3 \\
\rowcolor{blockred} \multirow{-5}{*}{\shortstack[l]{Gemma-2-9B-IT $\rightarrow$ \\ Gemma-2-2B-IT}}
& \texttt{CaRE-Div} & 16.3/22.1 & 10.7/\textbf{21.3} & 18.6/\textbf{19.0} & \textbf{19.5}/20.3 & \textbf{16.3/20.7} \\

\cdashline{1-7}

\rowcolor{blockblue}  & FKL & 23.9/31.4 & 17.8/25.1 & 29.2/30.7 & 17.1/22.1 & 22.0/27.3 \\
\rowcolor{blockblue}  & RKL & 26.5/36.5 & 19.0/29.8 & 34.2/37.6 & \textbf{20.0}/27.3 & 24.9/32.8 \\
\rowcolor{blockblue}  & SRKL & \textbf{28.2/39.4} & 20.0/\textbf{30.7} & 35.6/\textbf{38.2} & 19.1/\textbf{30.0} & \textbf{25.7/34.6} \\
\rowcolor{blockblue}  & $\alpha$-$\beta$ Div & 26.8/36.8 & 18.6/29.2 & 32.6/35.1 & 18.8/29.3 & 24.2/32.6 \\
\rowcolor{blockblue} \multirow{-5}{*}{\shortstack[l]{OpenLLaMA2-7B $\rightarrow$ \\ OpenLLaMA2-3B}}
& \texttt{CaRE-Div} & 27.4/36.9 & \textbf{20.2}/29.5 & \textbf{36.2}/38.0 & 18.5/28.3 & 25.6/33.2 \\

\bottomrule
\end{tabular}
}
\end{sc}
\end{small}
\end{table*}

\subsection{Comparison of \method\ against baselines}
\begin{table*}[t]
\caption{\textbf{Baseline comparison for GPT2-XL (1.5B) $\rightarrow$ GPT2-base (0.1B).} ROUGE-L on four held-out instruction benchmarks; values in parentheses are standard deviations across 5 seeds. \method\ achieves the best average ROUGE-L. Bold marks per-task best.}
\label{tab:baseline_gpt2base}
\centering
\begin{small}
\begin{sc}
\scalebox{1}{
\begin{tabular}{lccccc}
\toprule
\textbf{Method} & \textbf{Dolly} & \textbf{Self-Inst} & \textbf{Vicuna} & \textbf{Super-Nat} & \textbf{Avg.} \\
\midrule
SFT & 23.33 (0.22) & 10.56 (0.54) & 15.12 (0.47) & 17.08 (0.29) & 16.52 \\
KD~\citep{hinton2015distilling} & 23.52 (0.19) & 11.23 (0.41) & 15.92 (0.37) & 20.68 (0.14) & 17.84 \\
SeqKD~\citep{kim2016sequence} & 23.38 (0.37) & 10.18 (0.18) & 15.01 (0.28) & 15.08 (0.12) & 15.91 \\
ImitKD~\citep{lin2020autoregressiveknowledgedistillationimitation} & 21.63 (0.51) & 10.85 (0.38) & 14.70 (0.29) & 17.94 (0.10) & 16.28 \\
MiniLLM~\citep{gu2024minillm} & 23.84 (0.26) & 12.44 (0.28) & \bf 18.29 (0.36) & 22.62 (0.26) & 19.30 \\
GKD~\citep{agarwal2024policy} & 23.75 (0.15) & 12.73 (0.24) & 16.64 (0.24) & 23.05 (0.23) & 19.04 \\

Distillm~\citep{ko2024distillm} & 24.86 (0.16) & 12.11 (0.27) & 16.52 (0.53) & \textbf{23.62 (0.42)} & 19.27 \\
ABKD~\citep{wangabkd} & 23.93 (0.53) & 11.50 (0.30) & 16.11 (0.57) & 21.89 (0.33) & 18.36  \\
\cline{1-6}
\method (Ours) & \textbf{25.55 (0.24)} & \textbf{12.88 (0.09)} & 17.33 (0.57) & 22.80 (0.23) &  \textbf{19.64} \\
\bottomrule
\end{tabular}
}

\end{sc}
\end{small}
\end{table*}

\begin{table*}[t]
\caption{\textbf{Baseline comparison for OPT-2.7B $\rightarrow$ OPT-125M.} ROUGE-L on four held-out instruction benchmarks; values in parentheses are standard deviations across 5 seeds. \method\ achieves the best average ROUGE-L. Bold marks per-task best.}
\label{tab:baseline_opt}
\centering
\begin{small}
\begin{sc}
\scalebox{1}{
\begin{tabular}{lccccc}
\toprule
\textbf{Method} & \textbf{Dolly} & \textbf{Self-Inst} & \textbf{Vicuna} & \textbf{Super-Nat} & \textbf{Avg.} \\
\midrule
SFT & 21.78 (0.19) & 8.09 (0.39) & 14.40 (0.17) & 13.45 (0.20) & 14.43 \\
KD~\citep{hinton2015distilling} & 20.54 (0.38) & 9.16 (0.29) & 14.65 (0.47) & 15.79 (0.26) & 15.04 \\
SeqKD~\citep{kim2016sequence} & 20.72 (0.59) & 8.94 (0.37) & 13.56 (0.33) & 16.80 (0.36) & 15.01 \\
ImitKD~\citep{lin2020autoregressiveknowledgedistillationimitation} & 20.16 (0.19) & 8.95 (0.50) & 15.05 (0.52) & 14.72 (0.21) & 14.72 \\
MiniLLM~\citep{gu2024minillm} & 22.24 (0.32) & 9.92 (0.47) & \textbf{16.97 (0.49)} & 16.58 (0.19) & 16.43 \\
GKD~\citep{agarwal2024policy} & 22.46 (0.27) & 10.59 (0.36) & 16.25 (0.65) & 19.33 (0.20) & 17.16 \\
Distillm~\citep{ko2024distillm} & \textbf{25.07 (0.55)} & 11.60 (0.23) & 16.75 (0.23) & 22.45 (0.29) & 18.97 \\
ABKD~\citep{wangabkd} & 24.88 (0.45) & 11.18 (0.77) & 16.72 (0.24) & 21.73 (0.13) & 18.63  \\
\cline{1-6}
\method\ (Ours) & 24.33 (0.19) & \textbf{12.03 (0.20)} & 16.73 (0.29) & \textbf{23.93 (0.19)} & \textbf{19.25} \\
\bottomrule
\end{tabular}
}

\end{sc}
\end{small}
\end{table*}

\begin{table*}[t]
\caption{\textbf{Baseline comparison for OpenLLaMA2-7B $\rightarrow$ OpenLLaMA2-3B.} ROUGE-L on four held-out instruction benchmarks; values in parentheses are standard deviations across 5 seeds. \method\ is competitive with the strongest baseline (Distillm) on average and achieves the best score on Self-Instruct and Super-Natural Instructions. Bold marks per-task best.}
\label{tab:baseline_llama}
\centering
\begin{small}
\begin{sc}
\scalebox{1}{
\begin{tabular}{lccccc}
\toprule
\textbf{Method} & \textbf{Dolly} & \textbf{Self-Inst} & \textbf{Vicuna} & \textbf{Super-Nat} & \textbf{Avg.} \\
\midrule
SFT & 25.11 (0.58) & 16.52 (0.56) & 16.33 (0.39) & 29.28 (0.45) & 21.81 \\
KD~\citep{hinton2015distilling} & 20.95 (0.49) & 16.12 (0.80) & 15.39 (0.39) & 27.93 (0.26) & 20.10  \\
SeqKD~\citep{kim2016sequence} & 24.67 (0.47) & 15.83 (0.62) & 17.06 (0.47) & 29.06 (0.28) & 21.66 \\
ImitKD~\citep{lin2020autoregressiveknowledgedistillationimitation} & 24.53 (0.26) & 17.80 (0.80) & 17.36 (0.22) & 31.50 (0.12) & 22.80  \\
MiniLLM~\citep{gu2024minillm} & 27.88 (0.28) & 19.94 (0.58) & \bf 20.50 (0.41) & 36.91 (0.28) & 26.31  \\
GKD~\citep{agarwal2024policy} & 26.30 (0.31) & 19.56 (1.03) & 18.66 (0.34) & 35.71 (0.20) & 25.06 \\
Distillm~\citep{ko2024distillm} & \bf 28.97 (0.36) & 20.84 (0.54) & 19.19 (0.36) & 36.64 (0.15)  & \textbf{26.41}\\
ABKD~\citep{wangabkd} & 28.25 (0.35) & 19.25 (0.78) & 18.64 (0.26) & 33.53 (0.23) & 24.92 \\
\cline{1-6}
\method\ (Ours) & 27.74 (0.19) & \textbf{20.93 (0.90)} & 18.94 (0.55) & \textbf{37.34 (0.34)} & 26.23 \\
\bottomrule
\end{tabular}
}

\end{sc}
\end{small}
\end{table*}

We compare \method\ against a comprehensive suite of baselines, including standard approaches (SFT, KD, SeqKD) and advanced divergence-based methods (ImitKD, MiniLLM, GKD, Distillm, ABKD). Across all three teacher--student configurations, \method\ consistently achieves the strongest overall performance, confirming the effectiveness of confidence-gated optimization under noisy and heterogeneous teacher supervision.

\noindent \textbf{GPT-2 family (1.5B $\rightarrow$ 0.1B).}
In the GPT-2 setting, \method\ attains the highest average ROUGE-L of \textbf{19.64}, outperforming the strongest baselines (MiniLLM 19.30 and Distillm 19.27) with the clearest gains on Dolly (\textbf{25.55} vs.\ 24.86 for Distillm) and Self-Instruct (\textbf{12.88} vs.\ 12.73 for GKD). This result demonstrates that \method\ remains effective even under severe capacity mismatch, where the student is more than an order of magnitude smaller than the teacher. The improvement suggests that selectively adapting the optimization geometry and filtering unreliable supervision is especially beneficial when student capacity is limited.

\noindent \textbf{OPT family (2.7B $\rightarrow$ 125M).}
For OPT-based distillation, \method\ achieves the highest average ROUGE-L of \textbf{19.25} and outperforms standard KD and divergence-based baselines on Self-Instruct (\textbf{12.03}) and Super-Natural Instructions (\textbf{23.93}). In contrast, standard KD degrades performance below SFT due to noise propagation from uncertain teacher predictions (15.04 vs.\ 14.43). These results highlight the importance of the \revivalname\ mechanism for smaller students, which prevents overfitting to the teacher's epistemic uncertainty and allows the student to retain correct priors when the teacher is unreliable.

\noindent \textbf{OpenLLaMA family (7B $\rightarrow$ 3B).}
In the OpenLLaMA setting, \method\ directly addresses the fidelity trap observed in standard distillation, where KD (20.10) underperforms even the simple SFT baseline (21.81). By contrast, \method\ achieves an average ROUGE-L of \textbf{26.23}, competitive with the strongest baseline (Distillm at 26.41) within standard error and improving on it on both Self-Instruct (\textbf{20.93} vs.\ 20.84) and Super-Natural Instructions (\textbf{37.34} vs.\ 36.64). This indicates that \method\ effectively leverages informative teacher signals while avoiding over-alignment to noisy or hallucinated outputs, achieving comparable or better generalization than the strongest static-divergence baseline without specialized tuning.


\begin{figure*}[!htb]
    \centering
    \subfloat[With student-generated outputs]{\includegraphics[width=\linewidth]{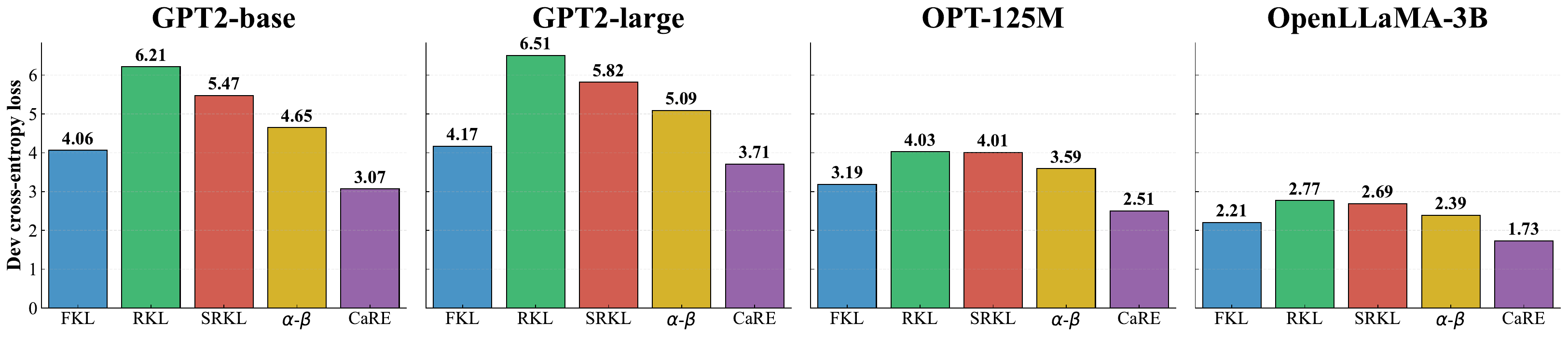}}
    \quad
    \subfloat[Without student-generated outputs]{\includegraphics[width=\linewidth]{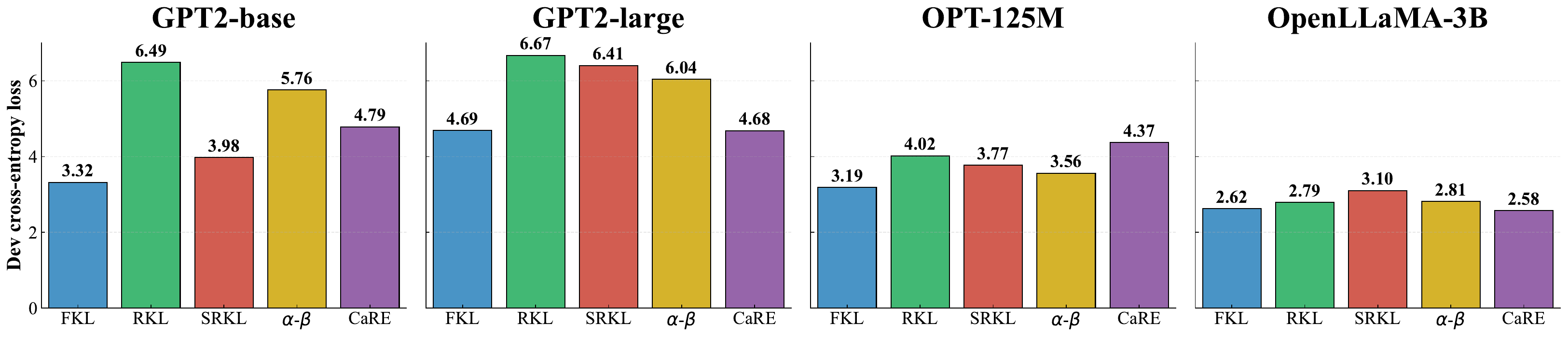}}
    \caption{Cross-entropy loss on the Dolly-15k development set.}
    \label{fig:valid_loss}
\end{figure*}

\begin{figure*}[!htb]
    \centering
    \subfloat[With student-generated outputs]{\includegraphics[width=\linewidth]{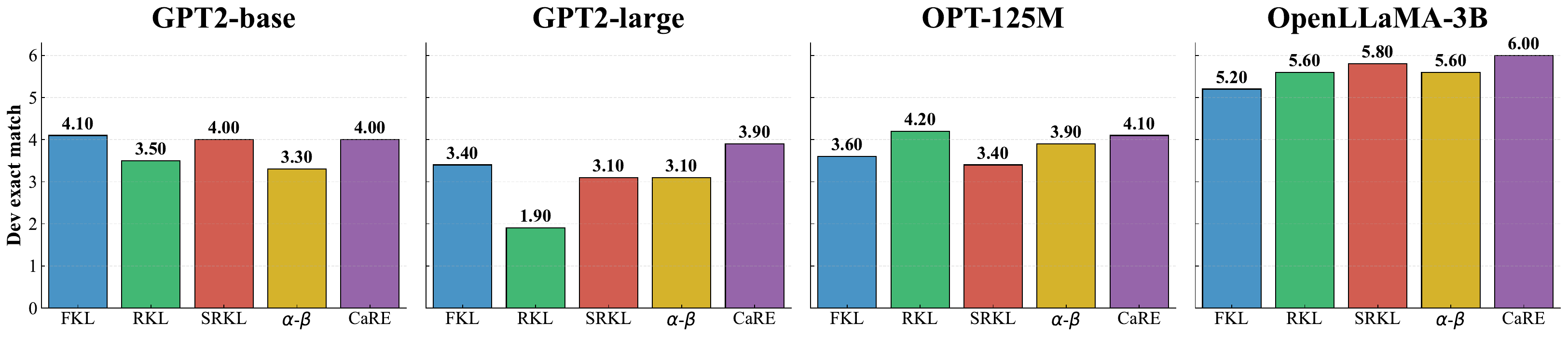}}
    \quad
    \subfloat[Without student-generated outputs]{\includegraphics[width=\linewidth]{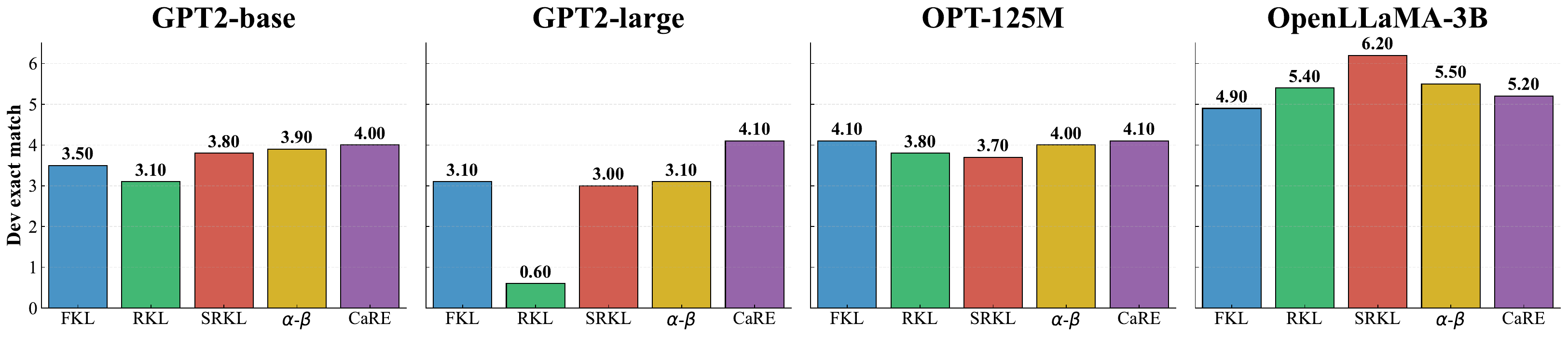}}
    \caption{Exact match (pass@1) accuracy on the Dolly-15k development set.}
    \label{fig:valid_exact_match}
\end{figure*}

Figure~\ref{fig:valid_loss} compares validation loss trajectories across divergence objectives. \loss\ consistently achieves the lowest validation loss, indicating a better fit to the target distribution and more stable optimization dynamics. For example, in the OpenLLaMA-3B setting, \loss\ converges to a validation loss of 1.73, compared to 2.21 for Forward KL and 2.69 for Skewed-RKL. This stability is most evident in the student-generated-output regime, where strictly mode-seeking objectives such as Reverse KL exhibit high variance and unstable convergence, with losses exceeding 4.8. In contrast, the confidence-gated mechanism of \loss\ maintains smooth and consistent convergence.

Figure~\ref{fig:valid_exact_match} reports exact match accuracy on the development set. In the GPT-2-Large experiment, \loss\ achieves exact match scores in the range 3.9--4.1, substantially outperforming Reverse KL (0.6--1.9) and $\alpha$--$\beta$ divergence (3.1). These results illustrate a key advantage of the bi-directional design: while purely reverse-divergence objectives often suffer from recall degradation due to aggressive mode collapse, \loss\ preserves precision without sacrificing recall, yielding robust performance across both small and large student models.


\begin{figure*}[!htb]
    \centering
    \includegraphics[width=\linewidth]{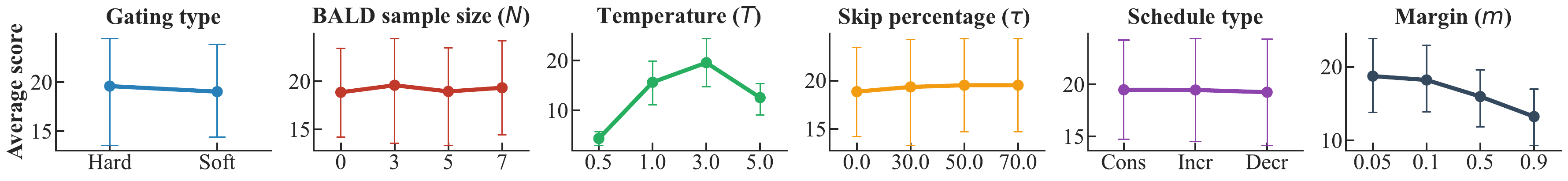}
    \caption{Ablation results with OPT-125M student.}
    \label{fig:ablation_opt}
\end{figure*}

We conduct an extensive ablation study on Dolly, Self-Instruct, Vicuna, and Super-Natural Instructions (Figures~\ref{fig:ablation} and~\ref{fig:ablation_opt}) to evaluate the robustness of \method\ with respect to its key design choices. We vary the aggressiveness of filtering (Skip \%), the \revivalname\ scheduling strategy, the number of Monte Carlo samples used for epistemic uncertainty estimation, the gating type, and the sampling temperature. Results are reported for both GPT-based and OPT-based student models, allowing us to assess architectural sensitivity and stability.

\noindent \textbf{Aggressiveness of filtering (Skip percentage).}
Across both GPT and OPT students, performance improves monotonically as the Skip percentage increases from low to moderate values, and then saturates. In particular, Skip values in the 50–70\% range consistently achieve the strongest results, while low filtering regimes (e.g., near 0\%) underperform. Notably, performance remains stable across a broad high-skip plateau rather than peaking sharply at a single value, indicating that \method\ is not brittle to the exact cutoff choice. This robustness supports our central claim: a substantial fraction of teacher supervision lies in epistemically unreliable regions, and selectively discarding it improves student performance without requiring fine-grained tuning.

\noindent \textbf{Impact of \revivalname\ scheduling.}
We compare constant (Con), increasing (Inc), and decreasing (Dec) schedules for the Skip percentage. Across all architectures and tasks, the increasing schedule consistently dominates, while the decreasing schedule performs worst. Constant schedules yield intermediate performance but are systematically inferior to increasing schedules. This ordering is stable across GPT and OPT models, confirming that the benefit is not architecture-specific. These results reinforce the curriculum interpretation: early training benefits from broader teacher supervision, whereas later stages require progressively stricter selectivity as the student's confidence improves.

\noindent \textbf{Epistemic estimation quality (number of MC samples).}
Varying the number of Monte Carlo dropout samples $N$ reveals a clear robustness trend. Increasing $N$ from 3 to 7 slightly smooths performance and improves stability, but the gains are modest. Importantly, strong performance is already achieved at $N=3$ when combined with an increasing schedule and moderate filtering. This indicates that \method\ does not rely on highly precise uncertainty estimates; instead, its design allows coarse epistemic signals to suffice, significantly reducing computational overhead while preserving performance.

\noindent \textbf{Gating type.}
We compare soft (continuous) and hard (binary) gating strategies. Soft gating exhibits slightly higher average performance across configurations, reflecting improved robustness to hyperparameter variation. However, hard gating consistently attains the highest peak scores in both GPT and OPT settings. This pattern reveals a desirable trade-off: soft gating offers stability across regimes, while hard gating maximizes gains when uncertainty estimates and schedules are well calibrated. Crucially, both gating strategies outperform static baselines, demonstrating that confidence adaptivity, rather than the exact gating form, is the dominant factor.

\noindent \textbf{Impact of sampling temperature.}
Performance as a function of sampling temperature follows a consistent unimodal pattern across architectures. Very low temperatures fail to induce sufficient stochastic diversity for reliable epistemic estimation, while very high temperatures inject excessive noise. Optimal performance is achieved at moderate temperatures (around $T=3.0$), with broad tolerance around this value. This further highlights the robustness of \method: effective uncertainty estimation does not require delicate tuning, only controlled stochasticity.

Taken together, these ablations demonstrate that \method\ is robust across a wide range of hyperparameter settings. Performance improvements arise from structural properties, confidence-gated geometry, and epistemic rejection, rather than fragile tuning. The consistency of trends across GPT and OPT families further confirms that the proposed framework generalizes across architectures, reinforcing its practical applicability for distilling modern LLMs.

\section{Detailed Discussions}
\label{app:discussion}

\begin{figure*}[!htb]
    \centering
    \subfloat[Skewed RKL without \revivalname]{\includegraphics[width=\linewidth]{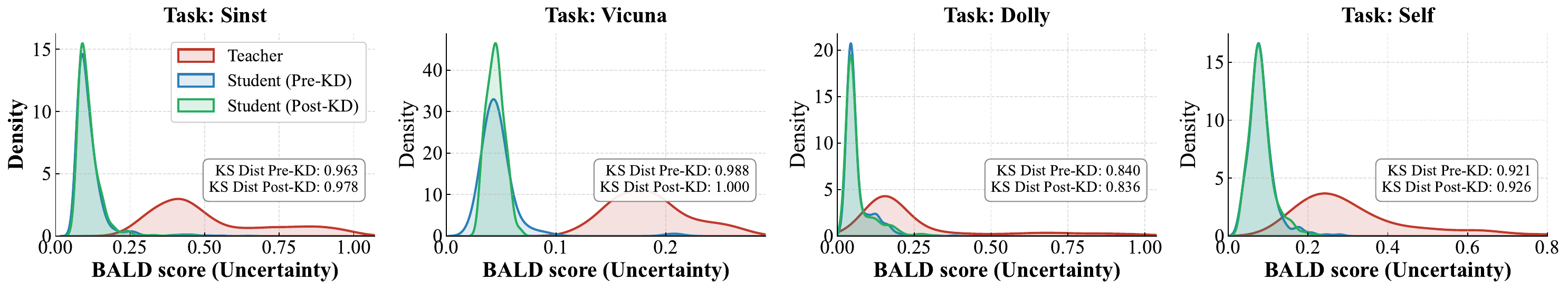}}
    \quad
    \subfloat[Skewed RKL with \revivalname]{\includegraphics[width=\linewidth]{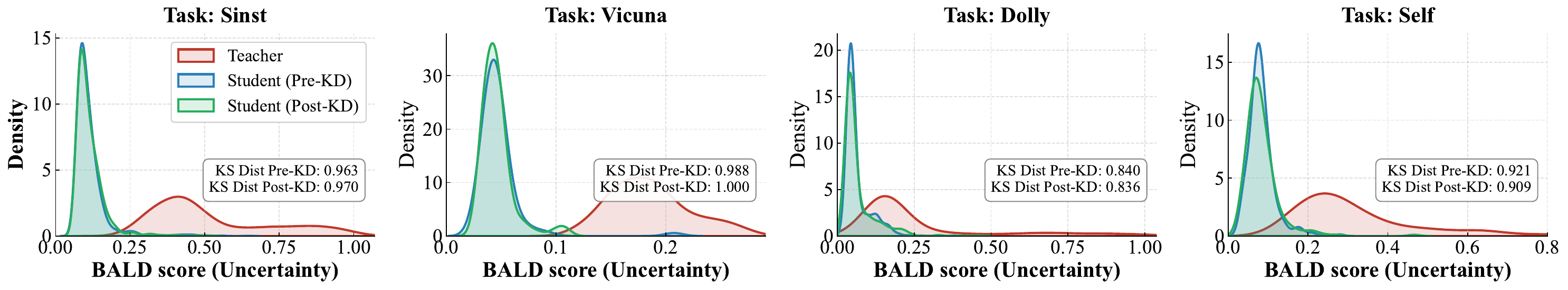}}
    \caption{Epistemic uncertainty (BALD) distributions of the OpenLLaMA-3B student under Skewed RKL, with and without \revivalname.}
    \label{fig:revival_bald_distillm2}
\end{figure*}

\begin{figure*}[!htb]
    \centering
    \subfloat[$\alpha$--$\beta$ divergence without \revivalname]{\includegraphics[width=\linewidth]{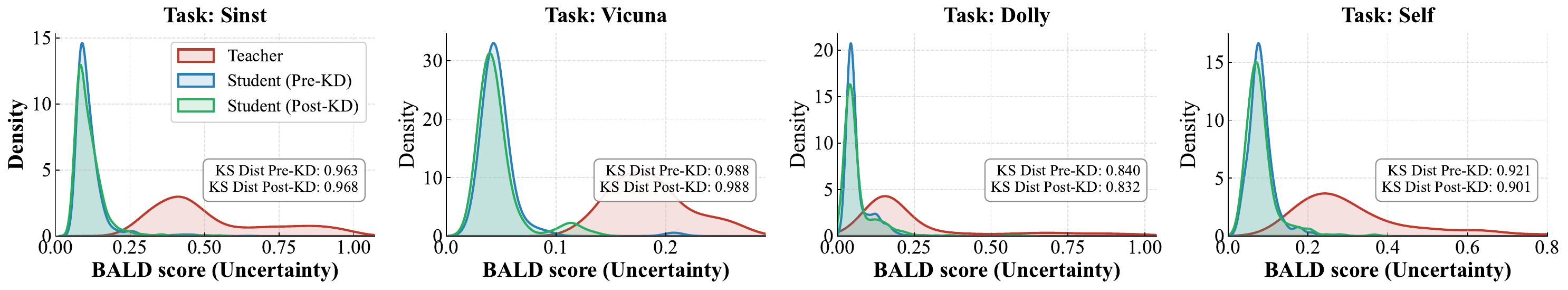}}
    \quad
    \subfloat[$\alpha$--$\beta$ divergence with \revivalname]{\includegraphics[width=\linewidth]{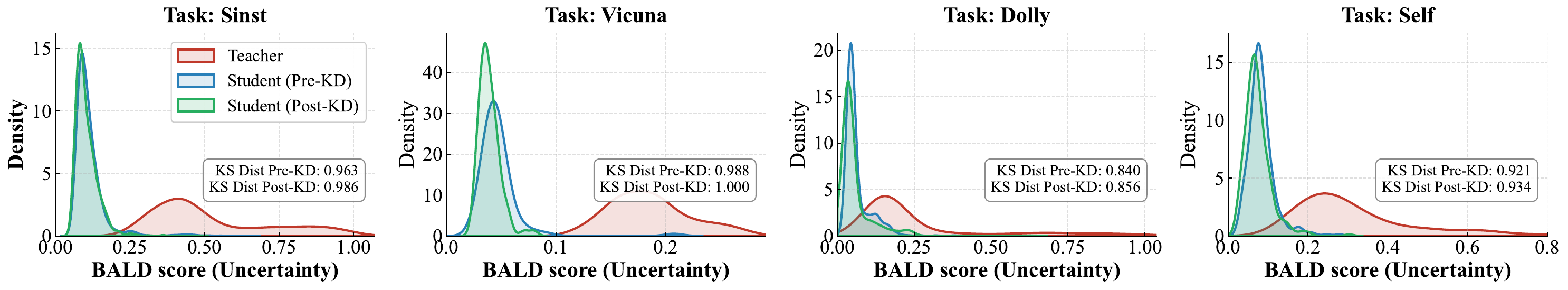}}
    \caption{Epistemic uncertainty (BALD) distributions of the OpenLLaMA-3B student under $\alpha$--$\beta$ divergence, with and without \revivalname.}
    \label{fig:revival_bald_ab}
\end{figure*}

\begin{figure*}[!htb]
    \centering
    \subfloat[GPT-2 Base]{\includegraphics[width=\linewidth]{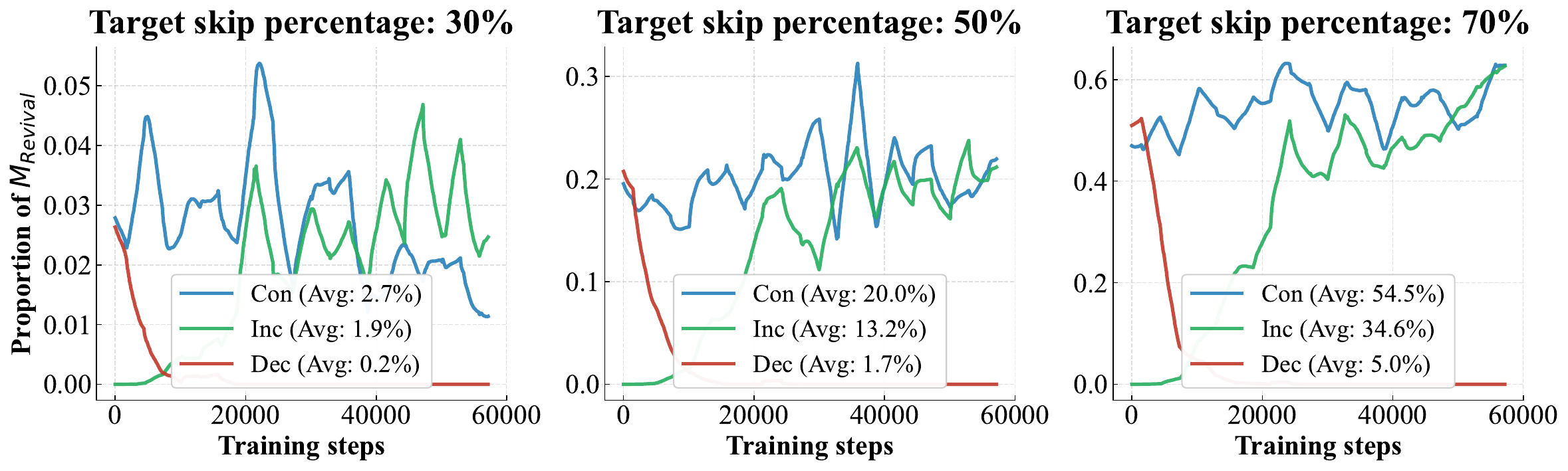}}
    \quad
    \subfloat[OPT-125M]{\includegraphics[width=\linewidth]{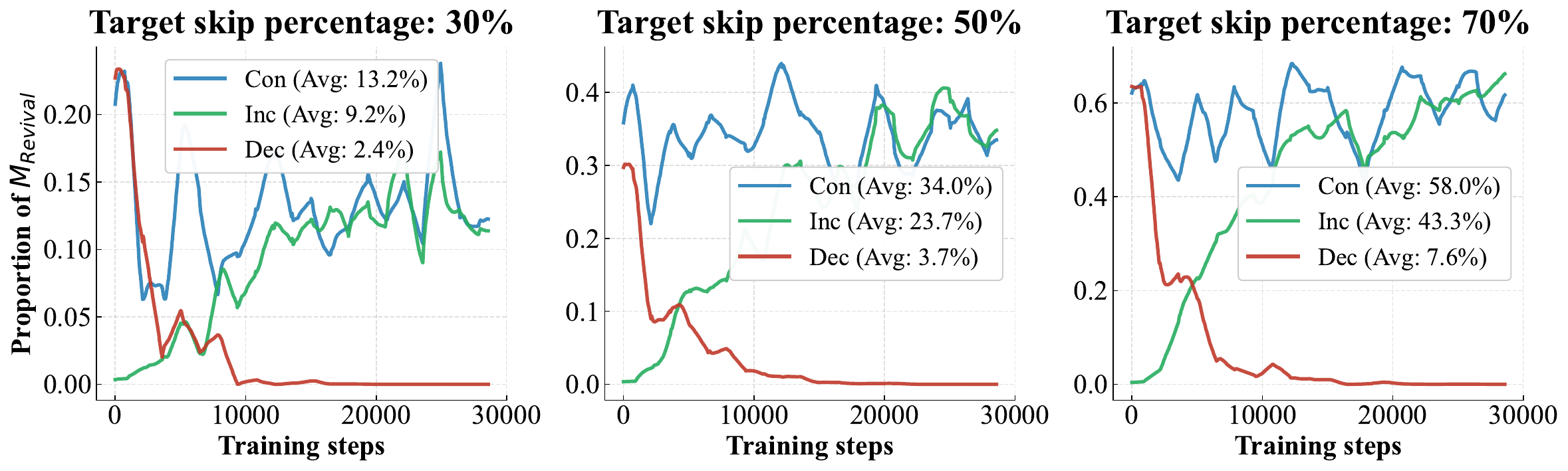}}
    \quad
    \subfloat[OpenLLaMA-3B]{\includegraphics[width=\linewidth]{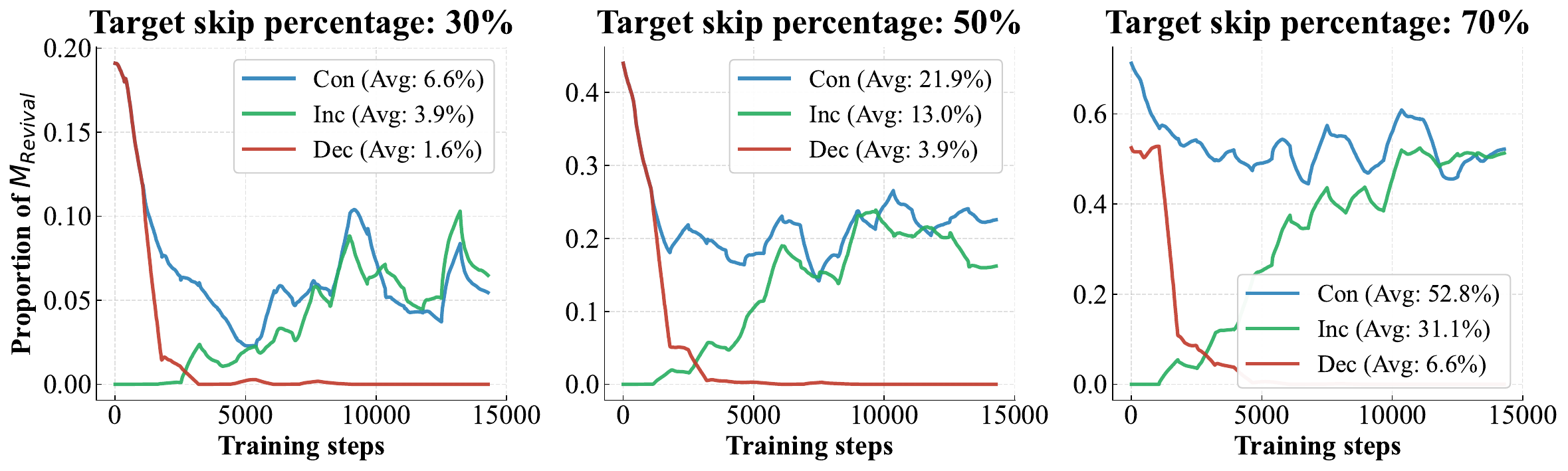}}
    \caption{Proportion of batches triggering $M_{\mathrm{Revival}}$ over training iterations for different skip targets and scheduling strategies.}
    \label{fig:conf_trend}
\end{figure*}

\begin{figure*}[!htb]
    \centering
    \subfloat[GPT-2 Base]{\includegraphics[width=0.33\linewidth]{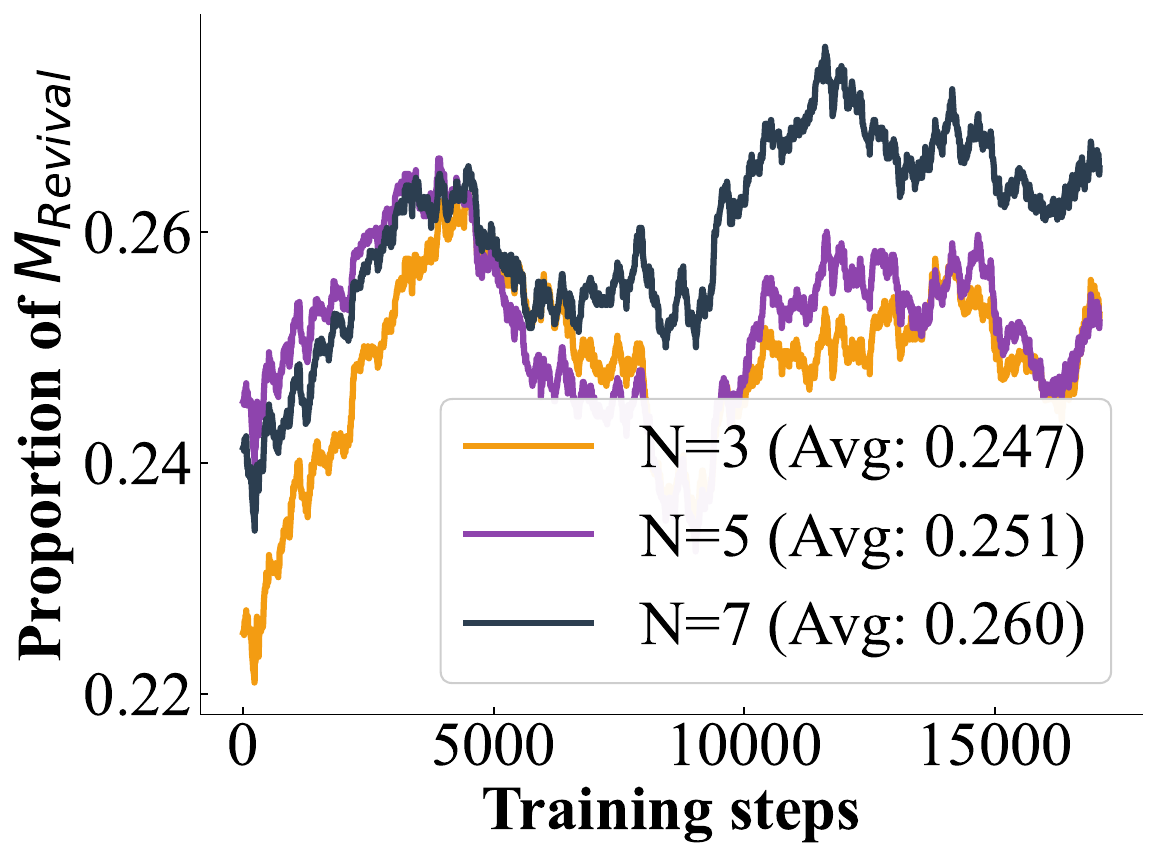}}
    \subfloat[OPT-125M]{\includegraphics[width=0.33\linewidth]{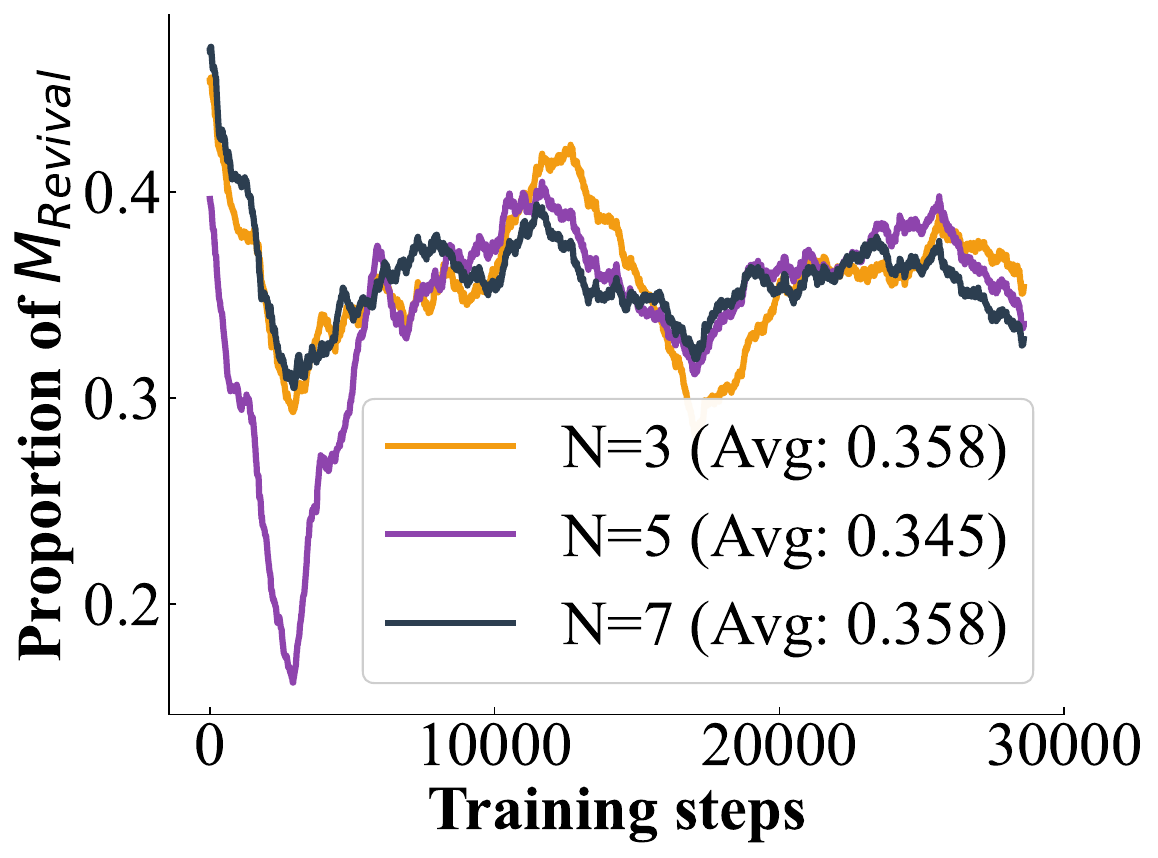}}
    \subfloat[OpenLLaMA-3B]{\includegraphics[width=0.33\linewidth]{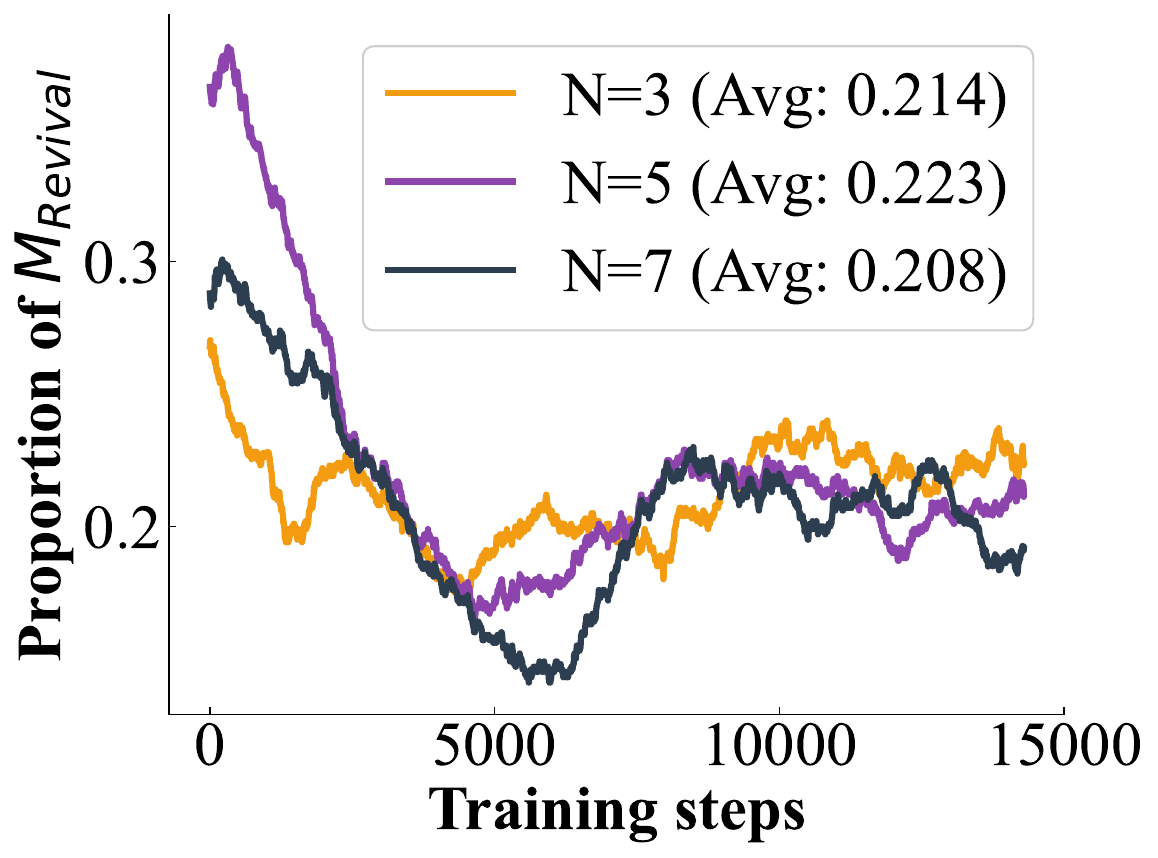}}
    \caption{Proportion of $M_{\mathrm{Revival}}$ activations for different numbers of MC dropout samples.}
    \label{fig:conf_trend2}
\end{figure*}

\subsection{Decoupling student confidence from teacher uncertainty}

A fundamental failure mode in knowledge distillation is that the student inherits not only the teacher's predictions but also its epistemic uncertainty, effectively learning to reproduce ambiguity rather than competence. This phenomenon is clearly visible when distillation is performed without selective rejection. Across instruction-following benchmarks, the post-distillation student's BALD distribution closely tracks the teacher's heavy-tailed uncertainty profile, exhibiting elevated mean uncertainty and limited separation from the teacher (Figure~\ref{fig:revival_bald}). Such behavior is consistent with the geometry of Forward KL and static divergence mixtures, which encourage the student to cover the teacher's full predictive mass, including regions dominated by epistemic noise.

In contrast, enabling \revivalname\ induces a systematic decoupling between teacher uncertainty and student confidence across all evaluated tasks. While teachers consistently exhibit broad, heavy-tailed BALD distributions, the student distilled with \revivalname\ maintains a sharply concentrated uncertainty profile that remains closer to its pre-distillation prior than to the teacher's ambiguity. This effect is reflected quantitatively in substantially larger teacher–student distributional separations, with Kolmogorov–Smirnov distances exceeding $0.5$ on benchmarks such as Vicuna and Dolly. These separations confirm that high-uncertainty teacher samples are actively filtered rather than partially absorbed.

Importantly, these results also expose the limitations of static divergence objectives. Although mode-seeking losses such as Skewed RKL and $\alpha$--$\beta$ divergence reduce variance in student predictions, they continue to enforce alignment on unreliable teacher tokens when applied uniformly. As a result, students trained under these objectives still inherit a non-trivial portion of the teacher's epistemic uncertainty. In contrast, \method\ combines adaptive divergence geometry with explicit epistemic rejection, allowing the student to ignore teacher supervision entirely in regimes where the teacher is unreliable. Consequently, the student's uncertainty distribution diverges more decisively from the teacher's, yielding sharper decision boundaries and higher distributional separation (Figures~\ref{fig:revival_bald_distillm2} and~\ref{fig:revival_bald_ab}).

Taken together, these observations clarify that effective distillation requires disentangling \emph{how} to learn (optimization geometry) from \emph{when} to learn (epistemic reliability). By enforcing this separation, \revivalname\ enables students that are both accurate and decisive, even when distilled from teachers that frequently hallucinate or exhibit high epistemic uncertainty.

\subsection{Evolution of student confidence over distillation iterations}

We further analyze the temporal evolution of student confidence by tracking the proportion of batches for which the \revivalname\ mask is activated, denoted by $M_{\mathrm{Revival}}$. By definition, $M_{\mathrm{Revival}}=1$ corresponds to the regime in which the teacher exhibits higher epistemic uncertainty than the student. Consequently, an increasing proportion of $M_{\mathrm{Revival}}$ over training directly reflects a growth in student epistemic confidence relative to the teacher.

Figure~\ref{fig:conf_trend} reports the evolution of $M_{\mathrm{Revival}}$ for three student architectures under different skip targets and scheduling strategies: constant (Con), increasing (Inc), and decreasing (Dec). Across all architectures, both constant and increasing schedules exhibit a clear upward trend in $M_{\mathrm{Revival}}$. Early in training, rejection is rare, indicating that the student is initially less confident than the teacher. As distillation proceeds, student epistemic uncertainty decreases, leading to a steady rise in $M_{\mathrm{Revival}}$ that eventually stabilizes near the target skip percentage.

This effect is most pronounced under the increasing schedule, where the gradual tightening of the rejection criterion induces a curriculum-like behavior. Early training benefits from broad teacher supervision, while later stages increasingly favor student confidence as competence improves. In contrast, decreasing schedules consistently suppress this effect. As the rejection criterion is relaxed over time, the student is re-exposed to high-uncertainty teacher supervision, resulting in limited growth in relative student confidence. This behavior mirrors the performance trends reported in Section~\ref{subsec:ablation_analysis}, where decreasing schedules yield inferior results.

Figure~\ref{fig:conf_trend2} examines the robustness of these dynamics with respect to the number of Monte Carlo dropout samples used to estimate BALD. While increasing the number of samples slightly smooths the trajectories, the qualitative behavior remains unchanged. In particular, the relative ordering of schedules is preserved across all architectures, confirming that the observed confidence evolution is not an artifact of stochastic estimation noise.

Overall, these results provide direct empirical evidence that the proposed framework induces a progressive increase in student confidence during distillation. Under constant and increasing schedules, the student transitions from reliance on teacher supervision to a regime in which it is epistemically more confident than the teacher on an expanding subset of data. This evolution aligns closely with the calibration dynamics formalized in Theorems~\ref{thm:conditional_calibration_overconfident} and~\ref{thm:conditional_sharpening_underconfident}, reinforcing our central claim that confidence-aware geometric adaptation combined with epistemic rejection yields students that are not only more accurate, but also increasingly decisive over the course of training.

\begin{figure}[!htb]
    \centering
    \subfloat[Chat alignment]{\includegraphics[width=0.33\linewidth]{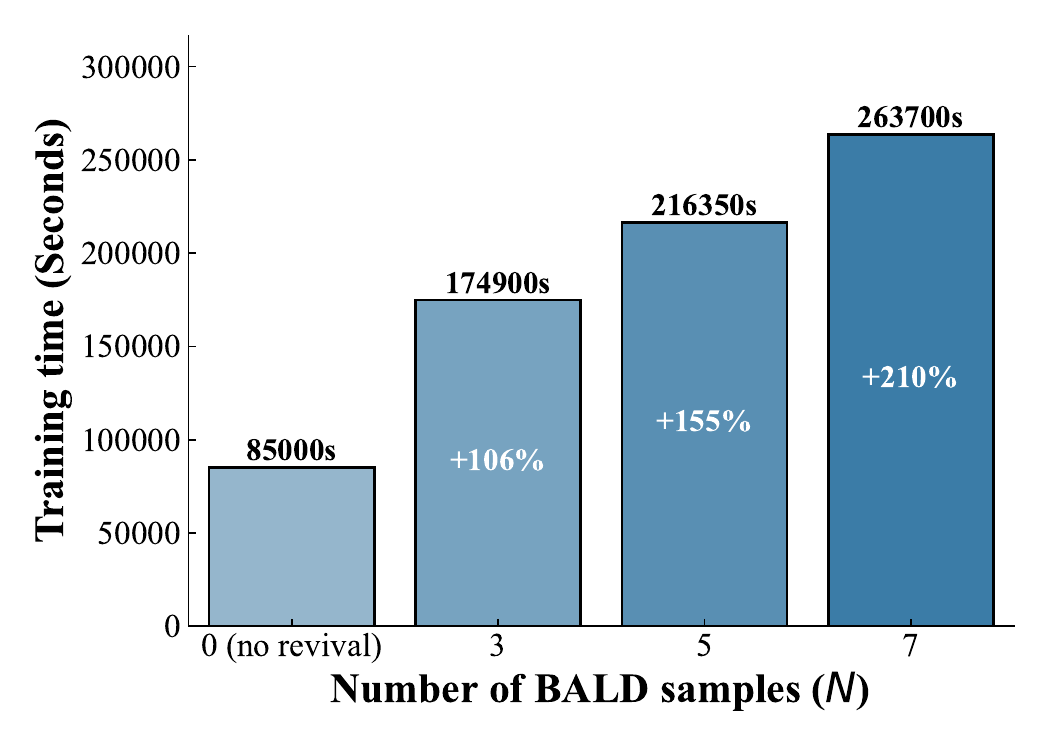}}
    \subfloat[Coding]{\includegraphics[width=0.33\linewidth]{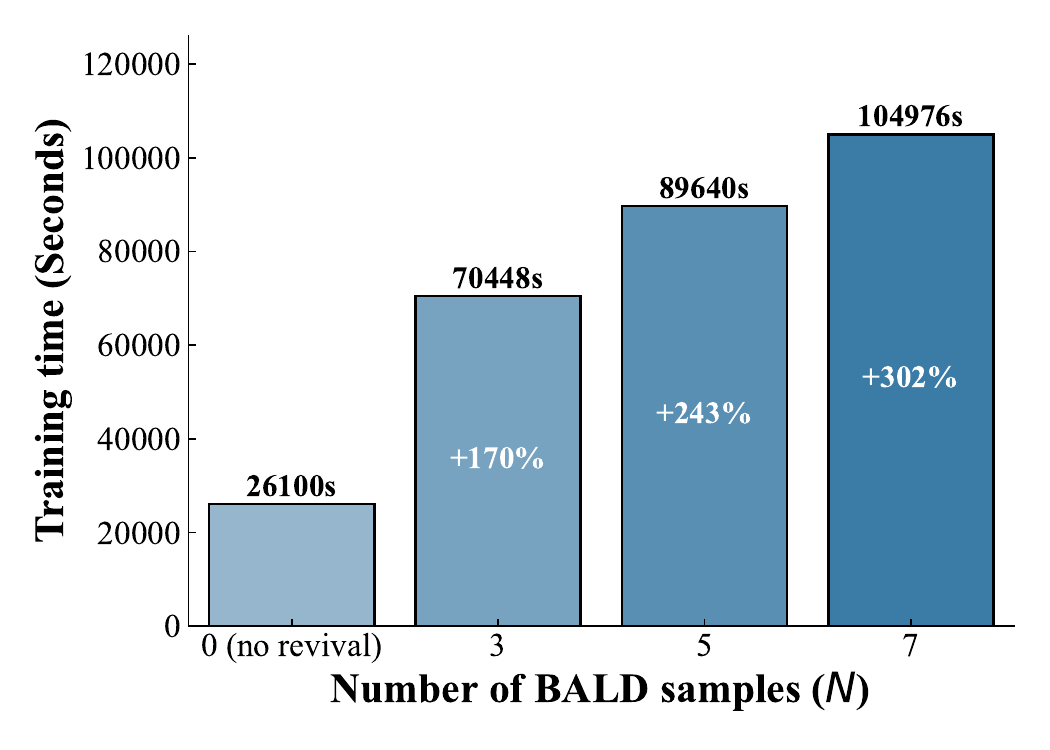}}
    \subfloat[Mathematical reasoning]{\includegraphics[width=0.33\linewidth]{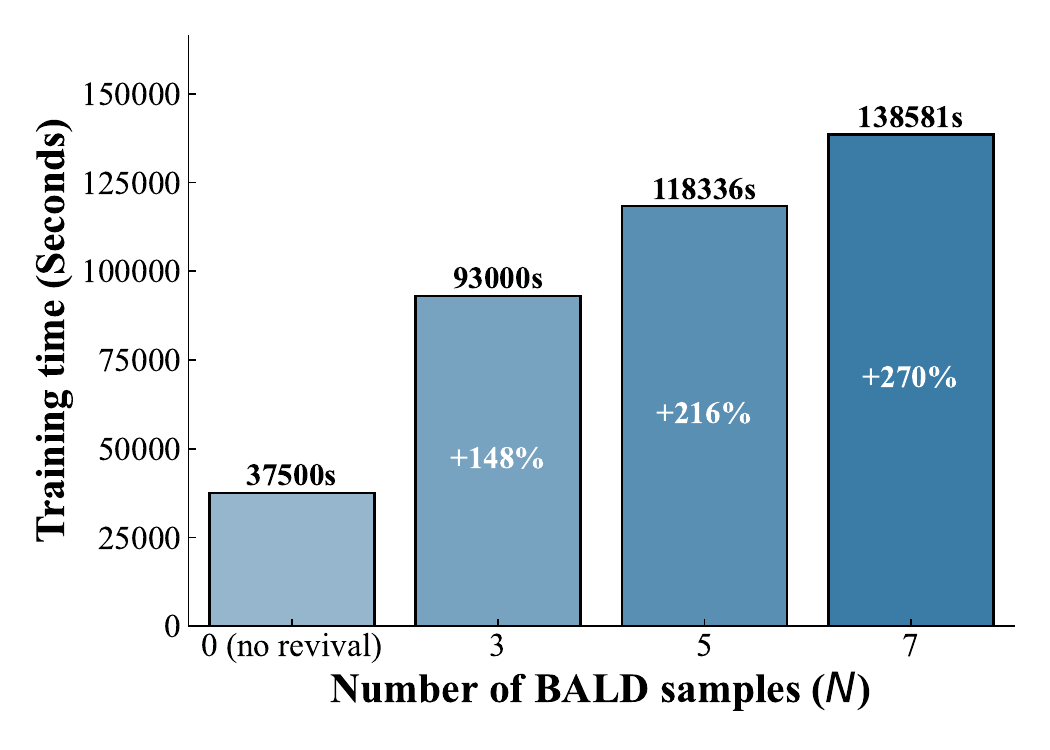}}
    \caption{\textbf{Runtime analysis} for different BALD sample configurations for specialized tasks.}
    \label{fig:runtime2}
\end{figure}

{\color{black}
\section{Calibration, Selective Prediction, and Selection-Bias Analyses}
\label{app:calibration_selection}

\subsection{Calibration (ECE) and selective prediction}
\label{app:calibration}

We report Expected Calibration Error \citep[ECE;][]{guo2017calibration} and risk--coverage / selective-prediction curves \citep{geifman2017selective} on the Dolly evaluation set for the GPT2 students. ECE must be read with care in this setting. Correctness is exact next-token match against a \emph{single} reference, so token accuracy is nearly identical for every objective ($0.43$--$0.47$), and the only quantity that really varies is how confident each student is. Ordered by average confidence, ECE increases monotonically across all six settings (Table~\ref{tab:ece}).

\begin{table}[h]
\centering
\color{black}
\caption{ECE on Dolly for the GPT2 students, ordered by average confidence. With a single-reference target, ECE tracks confidence almost one-for-one and is largely a proxy for sharpness.}
\label{tab:ece}
\begin{tabular}{lcc}
\toprule
\textbf{Objective} & \textbf{Avg.\ confidence} & \textbf{ECE} \\
\midrule
Forward KL & 0.67 & 0.22 \\
\method\ (softest configuration) & 0.68 & 0.28 \\
$\alpha$--$\beta$ divergence & 0.75 & 0.28 \\
Reverse KL & 0.77 & 0.31 \\
Skewed-RKL & 0.79 & 0.31 \\
\method\ (default temperature) & 0.79 & 0.34 \\
\bottomrule
\end{tabular}
\end{table}

\noindent ECE tracks average confidence monotonically, from Forward KL ($0.67$, ECE $0.22$) up to default \method\ ($0.79$, ECE $0.34$). On a single-reference target ECE is therefore largely a proxy for sharpness, and it penalizes exactly the precision-oriented, mode-seeking behavior our method is designed to produce. Forward KL reaches the lowest ECE only by being under-confident --- and it is also the weakest divergence on quality (lowest mean ROUGE-L in Table~\ref{tab:main_results}) --- while \method\ matches that same ECE level once its distribution is softened (softest configuration, confidence $0.68$). We therefore \emph{do not} claim that \method\ is the best-calibrated method by ground-truth ECE, and we do not regard ground-truth ECE as the right yardstick for a method whose aim is to keep the student sharp where it is competent.

The calibration property we actually establish is different. Theorem~\ref{thm:conditional_calibration_overconfident} states that the student's entropy converges toward the \emph{teacher's}, not toward a ground-truth-optimal value; the direct evidence is the uncertainty decoupling in Figures~\ref{fig:motivation}a and~\ref{fig:revival_bald} (a $7\%$ post-distillation KS drift on Dolly for \method\ against $0.4\%$ for Skewed-RKL). Separately, the risk--coverage curves are monotone for every objective, so a student's confidence is a usable selective-prediction signal regardless of its absolute scale.

\subsection{\revivalname\ does not preferentially reject hard examples or reduce diversity}
\label{app:selection_bias}

A quantile-scheduled rejection rule could in principle bias training by discarding hard examples or collapsing response diversity. We test this directly by comparing the region where \revivalname\ fires (teacher uncertain, student confident) against the complementary region on OpenLLaMA-Dolly, using LLM-judged question hardness and response diversity (Table~\ref{tab:selection_bias}). Neither difference is significant: \revivalname\ is not discarding the hard questions, and it does not measurably reduce the diversity of the student's responses.

\begin{table}[h]
\centering
\color{black}
\caption{Selection-bias check on OpenLLaMA-Dolly. LLM-judged question hardness and response diversity in the region where \revivalname\ fires vs.\ the complementary region. Neither difference is significant.}
\label{tab:selection_bias}
\begin{tabular}{lccc}
\toprule
\textbf{Measure (LLM-judged)} & \textbf{Region where \revivalname\ fires} & \textbf{Opposite region} & \textbf{Mann--Whitney $p$} \\
\midrule
Question hardness & 2.87 & 2.98 & 0.37 \\
Response diversity & 3.96 & 3.99 & 0.63 \\
\bottomrule
\end{tabular}
\end{table}
}

\end{document}